\documentclass{article}

\usepackage[dvipdfmx]{graphicx}
\usepackage{amsmath}
\usepackage{amssymb}
\usepackage{ascmac}
\usepackage{authblk}
\usepackage{bm}
\usepackage{caption}
\usepackage{color}
\usepackage[T1]{fontenc}
\usepackage[colorlinks=true]{hyperref}
\usepackage{mathtools}
\usepackage[numbers]{natbib}
\usepackage{stmaryrd}
\usepackage{tgtermes}

\usepackage{amsmath}
\usepackage{amssymb}
\usepackage{amsthm}
\usepackage{bm}
\usepackage{booktabs}
\usepackage[font=footnotesize,labelfont=bf]{caption}
\usepackage{color}
\usepackage{dsfont}
\usepackage[T1]{fontenc}
\usepackage{mathtools}
\usepackage{mathrsfs}
\usepackage{mdframed}
\usepackage{multirow}
\usepackage{pifont}
\usepackage{pgfplots}
\usepackage{rotating}
\usepackage[group-separator={,},group-minimum-digits={3}]{siunitx}
\usepackage{subcaption}
\usepackage{stmaryrd}
\usepackage{tabularx}
\usepackage{thmtools}
\usepackage{thm-restate}
\usepackage{tikz}
\usepackage{tikzsymbols}
\usepackage{wrapfig}

\hypersetup{
  colorlinks=true,
  linkcolor={red!15!green!35!blue!95},
  citecolor={green!50!black},
  urlcolor={red!50!black}
}

\newtheorem{theorem}{Theorem}
\newtheorem*{theorem*}{Theorem}
\newtheorem*{proof*}{proof}
\newtheorem{definition}{Definition}
\newtheorem*{definition*}{Definition}
\newtheorem{assumption}{Assumption}
\newtheorem*{assumption*}{Assumption}
\newtheorem{lemma}{Lemma}
\newtheorem*{lemma*}{Lemma}
\newtheorem{proposition}{Proposition}
\newtheorem*{proposition*}{Proposition}

\newtheorem*{corollary*}{Corollary}
\newtheorem{remark}{Remark}

\usetikzlibrary{calc}
\usetikzlibrary{fpu} 
\usetikzlibrary{external}
\newcommand{\Dcal}{{\mathcal{D}}}

\newcommand{\Ical}{{\mathcal{I}}}

\newcommand{\Lcal}{{\mathcal{L}}}
\newcommand{\Mcal}{{\mathcal{M}}}
\newcommand{\Ncal}{{\mathcal{N}}}
\newcommand{\Ocal}{{\mathcal{O}}}

\newcommand{\Ebb}{{\mathbb{E}}}

\newcommand{\Pbb}{{\mathbb{P}}}

\newcommand{\Rbb}{{\mathbb{R}}}
\newcommand{\Sbb}{{\mathbb{S}}}

\newcommand{\Abf}{{\mathbf{A}}}
\newcommand{\Bbf}{{\mathbf{B}}}
\newcommand{\Cbf}{{\mathbf{C}}}

\newcommand{\Fbf}{{\mathbf{F}}}
\newcommand{\Gbf}{{\mathbf{G}}}
\newcommand{\Hbf}{{\mathbf{H}}}
\newcommand{\Ibf}{{\mathbf{I}}}

\newcommand{\Kbf}{{\mathbf{K}}}
\newcommand{\Lbf}{{\mathbf{L}}}

\newcommand{\Obf}{{\mathbf{O}}}

\newcommand{\Ubf}{{\mathbf{U}}}

\newcommand{\Wbf}{{\mathbf{W}}}

\newcommand{\ubf}{{\mathbf{u}}}
\newcommand{\vbf}{{\mathbf{v}}}
\newcommand{\wbf}{{\mathbf{w}}}
\newcommand{\xbf}{{\mathbf{x}}}
\newcommand{\ybf}{{\mathbf{y}}}
\newcommand{\zbf}{{\mathbf{z}}}

\newcommand{\Lambdabf}{{\bm{\Lambda}}}

\newcommand{\Phibf}{{\bm{\Phi}}}

\newcommand{\Psibf}{{\bm{\Psi}}}

\newcommand{\omegabf}{{\bm{\omega}}}

\newcommand{\zerobf}{{\mathbf{0}}}

\newcommand{\defeq}{\coloneqq}

\DeclareMathOperator*{\E}{\Ebb}

\DeclareMathOperator{\diag}{\mathrm{diag}}
\DeclareMathOperator{\tr}{\mathrm{tr}}

\newcommand{\iverson}[1]{{\left\llbracket{#1}\right\rrbracket}}
\newcommand{\set}[1]{\left\lbrace{#1}\right\rbrace}
\newcommand{\setcomp}[2]{\left\lbrace{#1} \relmiddle| {#2}\right\rbrace}
\newcommand{\abs}[1]{\left|{#1}\right|}
\newcommand{\norm}[1]{\left\lVert{#1}\right\rVert}
\newcommand{\inpr}[2]{\left\langle{#1},{#2}\right\rangle}
\newcommand{\pinv}[1]{{{#1}^{\dagger}}}
\newcommand{\relmiddle}[1]{\mathrel{}\middle#1\mathrel{}}

\newcommand{\diff}[2]{\frac{\rd{#1}}{\rd{#2}}}

\newcommand{\rd}{\mathrm{d}}

\renewcommand{\epsilon}{\varepsilon}

\newcommand{\Prob}[1]{\Pbb\left\{{#1}\right\}}
\newcommand{\stochO}{{\Ocal_{\Pbb}}}
\newcommand{\stocho}{{o_{\Pbb}}}
\renewcommand{\vec}[1]{{\mathrm{vec}({#1})}}
\newcommand{\eigmin}[1]{{\lambda_{\min}({#1})}}

\newcommand{\stopgrad}{{\mathrm{SG}}}
\newcommand{\aug}{{\mathrm{aug}}}
\newcommand{\loss}{{\Lcal}}
\newcommand{\losssq}{{\loss_{\mathrm{sq}}}}

\newcommand{\losscos}{{\loss_{\mathrm{cos}}}}
\newcommand{\frob}{{\mathrm{F}}}

\newcommand{\sym}{{\mathbb{S}\mathrm{ym}}}

\usepackage[capitalize,nameinlink]{cleveref}
\Crefname{assumption}{Assumption}{Assumptions}
\crefname{assumption}{Assump.}{Assumps.}
\Crefname{theorem}{Theorem}{Theorems}
\crefname{theorem}{Thm.}{Thms.}
\Crefname{proposition}{Proposition}{Propsisions}
\crefname{proposition}{Prop.}{Props.}
\Crefname{lemma}{Lemma}{Lemmas}
\crefname{lemma}{Lem.}{Lems.}
\Crefformat{section}{\S#2#1#3}
\crefformat{section}{\S#2#1#3}
\Crefformat{appendix}{\S#2#1#3}
\crefformat{appendix}{\S#2#1#3}

\usepackage[top=30truemm,bottom=30truemm,left=25truemm,right=25truemm]{geometry}

\title{Feature Normalization Prevents Collapse of Non-contrastive Learning Dynamics}

\author{%
  Han Bao \\
  The Institute of Statistical Mathematics \\
  \texttt{bao.han@ism.ac.jp}
}

\begin{document}

\maketitle

\begin{abstract}
  Contrastive learning is a self-supervised representation learning framework, where two positive views generated through data augmentation are made similar by an attraction force in a data representation space, while a repulsive force makes them far from negative examples.
  Non-contrastive learning, represented by BYOL and SimSiam, further gets rid of negative examples and improves computational efficiency.
  While learned representations may collapse into a single point due to the lack of the repulsive force at first sight, \cite{Tian2021ICML} revealed through the learning dynamics analysis that the representations can avoid collapse if data augmentation is sufficiently stronger than regularization.
  However, their analysis does not take into account commonly-used \emph{feature normalization}, a normalizer before measuring the similarity of representations, and hence excessively strong regularization may still collapse the dynamics, which is an unnatural behavior under the presence of feature normalization.
  Therefore, we extend the previous theory based on the L2 loss by considering the cosine loss instead, which involves feature normalization.
  We show that the cosine loss induces sixth-order dynamics (while the L2 loss induces a third-order one), in which a stable equilibrium dynamically emerges even if there are only collapsed solutions with given initial parameters.
  Thus, we offer a new understanding that feature normalization plays an important role in robustly preventing the dynamics collapse.
\end{abstract}

\section{Introduction}
\label{section:introduction}

Modern machine learning often owes to the success of self-supervised representation learning,
contrastive learning is popular among them, in which data augmentation generates two positive views from the original data and their encoded features are contrasted with negative samples \citep{Chopra2005CVPR,vdOord2018}.
In particular, \cite{Chen2020ICML} conducted large-scale contrastive learning with 10K+ negative samples to establish comparable downstream classification performance even to supervised learners.
The benefit of large-scale negative samples has been observed both theoretically \citep{Nozawa2021NeurIPS,Bao2022ICML} and empirically \citep{Chen2021CVPR,Tomasev2022}, but it is disadvantageous in terms of computational efficiency.

By contrast, non-contrastive learning trains a feature encoder with only positive views, leveraging additional implementation tricks.
The seminal work \cite{Grill2020NeurIPS} proposed BYOL to introduce the momentum encoder and apply gradient stopping for one encoder branch only.
The follow-up work \cite{Chen2021CVPR} showed that gradient stopping brings success into non-contrastive learning via a simplified architecture SimSiam.
Despite their empirical successes, non-contrastive learning lacks the repulsive force induced by negative samples and learned representations may apparently end up with \emph{complete collapse} so that all points are mapped to a constant.
According to folklore, the success is attributed to asymmetric architectures between the two branches \citep{Wang2022CVPR}.
\cite{Tian2021ICML} first tackled the question \emph{why non-contrastive learning does not collapse}, by specifically studying the learning dynamics of BYOL.
They tracked the eigenvalues of the encoder parameters and found that the eigenvalue dynamics have non-trivial equilibriums unless the regularization is overly strong.
To put it differently, the balance between data augmentation and regularization controls the existence of non-trivial solutions.
However, this analysis dismisses \emph{feature normalization} practically added to normalize the encoded positive views before computing their similarity.
As feature normalization blows up when encoded features approach zero, the analysis of \cite{Tian2021ICML} may not fully explain the behavior of the non-contrastive learning dynamics with strong regularization.
Indeed, our pilot study (\cref{figure:pilot}) reveals that SimSiam learning dynamics does not collapse with much heavier regularization than the default strength $\rho=10^{-4}$.
The mechanism remains unclear why non-contrastive learning dynamically eschews collapse under heavy regularization.
\begin{figure}[t]
  \centering
  \includegraphics[width=0.6\textwidth]{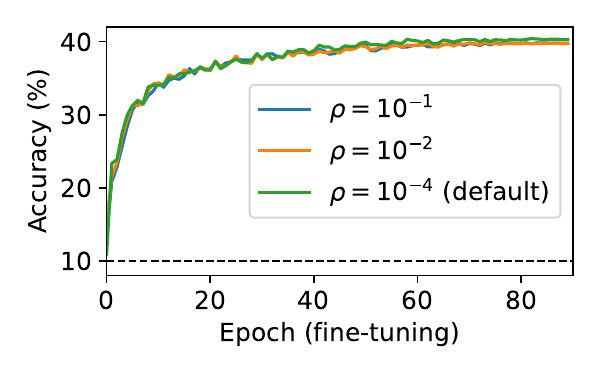}
  \caption{
    Linear probing accuracy of SimSiam representations of the CIFAR-10 dataset \citep{Krizhevsky2009techrep} is not driven to complete collapse by changing the weight decay intensity $\rho$---if the representation fell into complete collapse, the accuracy would stay at the chance rate ($10$\%, the horizontal dashed black line).
    The horizontal axis indicates fine-tuning epochs of the linear classifier.
    For non-contrastive pre-training, we used the ResNet-18 model \citep{He2016CVPR} with the initial learning rate $5\times 10^{-6}$, $500$ epochs, and different $\rho$ indicated in the legends.
    Other parameters and setup were inherited from the official implementation \citep{Chen2021CVPR}.
  }
  \label{figure:pilot}
\end{figure}

Therefore, we study the non-contrastive learning dynamics with feature normalization: an encoded feature $\Phibf\xbf$ for an input $\xbf \in \Rbb^d$ and encoder $\Phibf \in \Rbb^{h \times d}$ is normalized such as $\Phibf\xbf / \norm{\Phibf\xbf}_2$.
The main challenge is that the normalization yields a highly nonlinear dynamics because input random variables appear in the denominator of a loss, which makes the analysis of the expected loss convoluted.
This is a major reason why the existing studies on non-contrastive learning sticks to the L2-loss dynamics without the normalization \citep{Tian2021ICML,Wang2021arXiv,Pokle2022AISTATS,Wen2022NeurIPS,Liu2023ICLR,Tang2023ICML}.
Instead, we consider the high-dimensional limit $d, h \to \infty$, where the feature norm $\norm{\Phibf\xbf}_2$ concentrates around a constant regardless of $\xbf$, with proper initialization.
In this way, we can analyze the learning dynamics with feature normalization.
Under a synthetic setup, we derive the learning dynamics of encoder parameters (\cref{section:dynamics}), and disentangle it into the eigenvalue dynamics (\cref{section:eigenvalue_dynamics}).
The eigenvalue dynamics is sixth-order, and we find that a stable equilibrium emerges even if there is no stable equilibrium with the initial parametrization and regularization strength (\cref{section:regimes}).
This dynamics behavior is in contrast to the third-order dynamics of \cite{Tian2021ICML}, compared in \cref{section:comparison}.
We additionally observe how a stable equilibrium emerges through numerical simulation (\cref{section:experiments}).
Thus, we demonstrate how feature normalization prevents the complete collapse using a synthetic model, without resorting to the L2 loss formulation.

\section{Related work}
\label{section:related}

Recent advances in contrastive learning can be attributed to the InfoNCE loss \citep{vdOord2018}, which can be regarded as a multi-sample mutual information estimator between the two views \citep{Poole2019ICML,Song2020ICLR}.
\cite{Chen2020ICML} showed that large-scale contrastive representation learning can potentially perform comparably to supervised vision learners.
This empirical success owes to a huge number of negative samples, forming a repulsive force in contrastive learning.
Follow-up studies confirmed that larger negative samples are generally beneficial for downstream performance \citep{Chen2021CVPR,Tomasev2022},
and the phenomenon has been verified through theoretical analysis of the downstream classification error \citep{Nozawa2021NeurIPS,Wang2022ICLR,Bao2022ICML,Awasthi2022ICML},
whereas larger negative samples require heavier computation.

Non-contrastive learning is another framework without negative samples.
Although it may fail due to lack of the repulsive force,
BYOL \citep{Grill2020NeurIPS} and SimSiam \citep{Chen2021CVPR} introduced clever implementation tricks based on Siamese nets.
Other approaches conduct representation learning and clustering iteratively (e.g., SwAV \citep{Caron2020NeurIPS} and TCR \citep{Li2022}),
impose regularization on the covariance matrix (e.g., Barlow Twins \citep{Zbontar2021ICML}, W-MSE \citep{Ermolov2021ICML}, and VICReg \citep{Bardes2022ICLR}),
and leverage distillation (e.g., DINO \citep{Caron2021ICCV}).
While these methods succeed, we are still seeking theoretical understanding of \emph{why} non-contrastive dynamics does not collapse and \emph{what} non-contrastive dynamics learns.
For the latter, recent studies revealed that it implicitly learns a subspace \citep{Wang2021arXiv}, sparse signals \citep{Wen2021ICML}, a permutation matrix over latent variables \citep{Pokle2022AISTATS}, augmentation-invariant equivalent classes \citep{Dubois2022NeurIPS}, and a low-pass filter of parameter spectra \citep{Zhuo2023ICLR}.
Besides, contrastive supervision is theoretically useful for downstream classification under a simplified setup \citep{Bao2018ICML,Bao2022AISTATS}.

How does non-contrastive dynamics avoid collapse?
The seminal work \cite{Tian2021ICML} analyzed the BYOL/SimSiam dynamics and found that data augmentation behaves as a repulsive force to prevent eigenvalues of network parameters from collapsing unless regularization is not excessively strong.
We closely follow them and extend it to incorporate feature normalization.
Independently from us, non-contrastive learning dynamics has been studied:
\cite{Wen2022NeurIPS} revealed that the dynamics without the predictor shall collapse, but only its off-diagonal elements are assumed to be trainable.
\cite{Zhang2022ICLR} hypothesized based on their pilot study that an extra gradient term in the dynamics alleviates the dimensional collapse, which has not been formalized yet.
\cite{Wang2022} and \cite{Halvagal2023NeurIPS} tackled dynamics with feature normalization similar to us but with dynamics of representations instead of network parameters, which changes the training dynamics---parameters are learned but not representations directly.
Nonetheless, the latter \citep{Halvagal2023NeurIPS} revealed an interesting implicit bias so that non-zero eigenvalues converges closely to each other.
Let us mention a couple of studies on different collapse phenomena:
\cite{Richemond2023ICML} revealed that the BYOL predictor gradually increases \emph{rank} but with a fixed target net.
\cite{Li2022ECCV} revealed that downstream performance is predictable from the degree of \emph{dimensional} collapse.
\cite{Balestriero2022NeurIPS} and \cite{Liu2023ICLR} showed that non-contrastive dynamics including VICReg may cause \emph{dimensional} collapse based on the loss landscape.
To wrap up, we believe that studying \emph{parameter} dynamics is a direct approach towards understanding \emph{complete} collapse (i.e., all eigenvalues $\to 0$), and incorporating feature normalization is an important piece to get closer to practices.

Lastly, we mention a few articles studying normalization from different perspectives.
Whereas we focus on non-trivial equilibriums in self-supervised learning dynamics, \cite{Dukler2020ICML} and \cite{Wan2021NeurIPS} studied the convergence of general gradient descent with weight normalization, and \cite{Joudaki2023arXiv} studied how normalization prevents rank collapse of nonlinear MLPs at the infinite-depth limit via isometry.

\section{Model and loss functions}
\label{section:model}

\paragraph{Notations.}
The $n$-dimensional Euclidean space and hypersphere are denoted by $\Rbb^n$ and $\Sbb^{n-1}$, respectively.
The L2, Frobenius, and spectral norms are denoted by $\norm{\cdot}_2$, $\norm{\cdot}_\frob$, and $\norm{\cdot}$, respectively.
The $n \times n$ identity matrix is denoted by $\Ibf_n$,
or by $\Ibf$ whenever clear from the context.
For $\ubf, \vbf \in \Rbb^n$, $\inpr{\ubf}{\vbf} = \ubf^\top\vbf$ denotes the inner product.
For $\Abf, \Bbf \in \Rbb^{n_1 \times n_2}$, $\inpr{\Abf}{\Bbf}_\frob = \sum_{i,j} A_{i,j} B_{i,j}$ denotes the Frobenius inner product.
For a time-dependent matrix $\Abf$, we explicitly write $\Abf(t)$ if necessary.
The Moore--Penrose inverse of a matrix $\Abf$ is denoted by $\pinv{\Abf}$.
The set of $n \times n$ symmetric matrices is denoted by $\sym_n$.
The upper and lower asymptotic orders are denoted by $\Ocal(\cdot)$ and $\Omega(\cdot)$, respectively.
The stochastic orders indexed by $h$ are denoted by $\stochO(\cdot)$ and $\stocho(\cdot)$, respectively.

\paragraph{Model.}
We focus on SimSiam \citep{Chen2021CVPR} as a non-contrastive learner.
We first sample a $d$-dimensional anchor input $\xbf_0 \sim \Dcal$ and augment to two views $\xbf, \xbf' \sim \Dcal^\aug_{\xbf_0}$,
where $\Dcal^\aug_{\xbf_0}$ is the augmentation distribution.
While affine transforms or random maskings of images are common augmentations \citep{Chen2020ICML,He2022CVPR},
we assume the isotropic Gaussian augmentation $\Dcal^\aug_{\xbf_0} = \Ncal(\xbf_0, \sigma^2 \Ibf)$
to simplify, and let $\sigma^2$ be its intensity.
For the input distribution, we suppose the multivariate Gaussian $\Dcal = \Ncal(\zerobf, \Ibf)$
to devote ourselves to understanding dynamics, as in \cite{Saxe2014ICLR} and \cite{Tian2021ICML}.

Our network encoder consists of two layers:
representation net $\Phibf \in \Rbb^{h \times d}$ and projection head $\Wbf \in \Rbb^{h \times h}$,
where $h$ is the representation dimension.
For the two views $\xbf, \xbf'$, we obtain \emph{online} $\Phibf\xbf \in \Rbb^h$ and \emph{target} representation $\Phibf\xbf' \in \Rbb^h$,
and predict the target from the online representation by $\Wbf\Phibf\xbf \in \Rbb^h$.
Here, we ablate the exponential moving average used in BYOL for simplicity.

\paragraph{Loss functions.}
BYOL/SimSiam introduce \emph{asymmetry} of the two branches with the stop gradient operator, denoted by $\stopgrad(\cdot)$,
where parameters are regarded as constants during backpropagation \citep{Chen2021CVPR}.
\cite{Tian2021ICML} used the following \emph{L2 loss} to describe non-contrastive dynamics:
\begin{equation}
  \losssq(\Phibf, \Wbf) \defeq \frac{1}{2} \E_{\xbf_0} \E_{\xbf, \xbf' \mid \xbf_0}[\norm{\Wbf\Phibf\xbf - \stopgrad(\Phibf\xbf')}_2^2],
  \label{equation:l2_loss}
\end{equation}
where the expectations are taken over $\xbf, \xbf' \sim \Dcal^\aug_{\xbf_0}$ and $\xbf_0 \sim \Dcal$.
Thanks to the closed-form solution, the L2 loss has been prevailing in the existing analyses \citep{Wang2021arXiv,Tang2023ICML,Zhuo2023ICLR}.

We instead focus on the following \emph{cosine loss} to take feature normalization into account,
which is a key factor in the success of contrastive representation learning \citep{Wang2020ICML}:
\begin{equation}
  \losscos(\Phibf, \Wbf) \defeq \E_{\xbf_0} \E_{\xbf, \xbf' \mid \xbf_0}\left[-\frac{\inpr{\Wbf\Phibf\xbf}{\stopgrad(\Phibf\xbf')}}{\norm{\Wbf\Phibf\xbf}_2 \norm{\stopgrad(\Phibf\xbf')}_2}\right].
  \label{equation:cos_loss}
\end{equation}
Importantly, the cosine loss has been used in most practical implementations \citep{Grill2020NeurIPS,Chen2021CVPR},
including a reproductive research \citep{Hoppe2022} of simulations in \cite{Tian2021ICML}.
We can easily confirm that BYOL/SimSiam immediately collapse if we use the L2 loss experimentally.
Subsequently, the weight decay $R(\Phibf, \Wbf) \defeq \frac{\rho}{2}(\norm{\Phibf}_\frob^2 + \norm{\Wbf}_\frob^2)$ is added with a regularization strength $\rho > 0$.

\section{Non-contrastive dynamics in proportional limit}
\label{section:dynamics}

Let us focus on the cosine loss and derive its non-contrastive dynamics via the gradient flow.
See \cref{section:proofs} for the proofs of lemmas provided subsequently.
As the continuous limit of the gradient descent where learning rates are taken to be infinitesimal \citep{Saxe2014ICLR},
we characterize time evolution of the network parameters by the following simultaneous ordinary differential equation:
\begin{equation}
  \label{equation:gradient_flow}
  \dot\Phibf = - \nabla_\Phibf \{\losscos(\Phibf, \Wbf) + R(\Phibf, \Wbf)\}, \quad
  \dot\Wbf = - \nabla_\Wbf \{\losscos(\Phibf, \Wbf) + R(\Phibf, \Wbf)\}.
\end{equation}

To derive the dynamics, several assumptions are imposed.
\begin{assumption}[Symmetric projection]
  \label{assumption:w_symmetric}
  $\Wbf \in \sym_h$ holds during time evolution.
\end{assumption}
\begin{assumption}[Input distribution]
  \label{assumption:standard_mvn}
  $\Dcal = \Ncal(\zerobf, \Ibf)$.
\end{assumption}
\begin{assumption}[Proportional limit]
  \label{assumption:limit}
  $d, h \to \infty$, and $d / h \to \alpha$ for some $\alpha \in (0, \infty)$.
\end{assumption}
\begin{assumption}[Parameter initialization]
  \label{assumption:initialization}
  $\Phibf$ is initialized with $\sqrt{d} \cdot \Phibf(0)_{ij} \sim \Ncal(0, 1)$ for $i \in [h], j \in [d]$.
  $\Wbf$ is initialized with $\sqrt{h} \cdot \Wbf(0)_{ij} \sim \Ncal(0, 1)$ for $i, j \in [h]$.
\end{assumption}
\Cref{assumption:w_symmetric,assumption:standard_mvn} are inherited from \cite{Tian2021ICML} and make subsequent analyses transparent.
We empirically verify that the non-contrastive dynamics reasonably maintains the symmetry of $\Wbf$ during the training later (\cref{section:experiments}).
\Cref{assumption:limit} is a cornerstone to our analysis: the high-dimensional limit makes Gaussian random vectors concentrate on a sphere, which leads to a closed-form solution for the cosine loss dynamics.
We suppose that the hidden unit size $h=512$ (used in SimSiam) suffices.%
\footnote{
  The high-dimensional limit is used merely for invoking concentration inequalities, where $h=512$ suffices to control tail probabilities because the Orlicz norms usually remain moderately large like $2$, for example in \cref{lemma:frobenius_norm_bound}.
  Note, however, that representations at the high-dimensional limit would be arguable with the low-dimensional manifold assumption being in one's mind.
}
\Cref{assumption:initialization} is a standard initialization scale empirically in the He initialization \citep{He2015ICCV} and theoretically in the neural tangent kernel regime \citep{Jacot2018NeurIPS}.
This initialization scale maintains norms of the random matrices $\Phibf$ and $\Wbf\Phibf$ without vanishing or exploding under the proportional limit.

\begin{restatable}{lemma}{thmdynamics}
  \hyperlink{proof:dynamics_matrix_expected}{(proof $\blacktriangledown$)}
  \label{lemma:dynamics_matrix_expected}
  Parameter matrices $\Wbf$ and $\Phibf$ evolve as follows:
  \begin{equation}
    \label{equation:dynamics_matrix_expected}
    \Wbf^\top\dot\Wbf = \Hbf -\rho\Wbf\Wbf^\top, \quad
    \dot\Phibf\Phibf^\top\Wbf^\top = \Wbf^\top\Hbf - \rho\Phibf\Phibf^\top\Wbf^\top,
  \end{equation}
  where $\Hbf \defeq \E[\zbf'\omegabf^\top - (\omegabf^\top\zbf')\omegabf\omegabf^\top]$, $\zbf' \defeq \Phibf\xbf' / \norm{\Phibf\xbf'}_2$, and $\omegabf \defeq \Wbf\Phibf\xbf / \norm{\Wbf\Phibf\xbf}_2$.
  The expectation in $\Hbf$ is taken over $\xbf_0, \xbf$, and $\xbf'$.
\end{restatable}
\Cref{equation:dynamics_matrix_expected} is derived from \cref{equation:gradient_flow} with the standard matrix calculus.
We will analyze \cref{equation:dynamics_matrix_expected} to see when the dynamics stably converges to a non-trivial solution.
To solve it, we need to evaluate $\Hbf$ first.
This involves expectations with $\zbf'$ and $\omegabf$,
which are normalized Gaussian vectors and cannot be straightforwardly evaluated.
Here, we take a step further by considering the proportional limit (\cref{assumption:limit}),
where norms of Gaussian vectors are concentrated.
This regime allows us to directly evaluate Gaussian random vectors instead of the normalized ones.
\begin{restatable}{lemma}{thmconci}
  \hyperlink{proof:norm_concentration}{(proof $\blacktriangledown$)}
  \label{lemma:norm_concentration}
  Let $\Psibf \defeq \Wbf\Phibf$.
  Under \cref{assumption:w_symmetric,assumption:initialization,assumption:standard_mvn,assumption:limit}, for a fixed $\xbf_0$, the norms of $\Phibf\xbf$ and $\Wbf\Phibf\xbf$ (as well as $\Phibf\xbf'$ and $\Wbf\Phibf\xbf'$) are concentrated:
  \begin{align*}
      \norm{\tfrac{1}{\sqrt{h\sigma^2}}\Phibf\xbf}_2^2 &= \norm{\tfrac{1}{\sqrt{h}}\Phibf}_\frob^2 + \norm{\tfrac{1}{\sqrt{h\sigma^2}}\Phibf\xbf_0}_2^2 + \stocho(1), \\
      \norm{\tfrac{1}{\sqrt{h^2\sigma^2}}\Psibf\xbf}_2^2 &= \norm{\tfrac{1}{\sqrt{h^2}}\Psibf}_\frob^2 + \norm{\tfrac{1}{\sqrt{h^2\sigma^2}}\Psibf\xbf_0}_2^2 + \stocho(1).
  \end{align*}
\end{restatable}
\begin{restatable}{lemma}{thmconcii}
  \hyperlink{proof:norm_concentration_x}{(proof $\blacktriangledown$)}
  \label{lemma:norm_concentration_x}
  Let $\Psibf \defeq \Wbf\Phibf$.
  Under \cref{assumption:w_symmetric,assumption:initialization,assumption:standard_mvn,assumption:limit}, the following concentrations are established:
  \begin{equation*}
    \norm{\tfrac{1}{\sqrt{h\sigma^2}}\Phibf\xbf_0}_2 = \norm{\tfrac{1}{\sqrt{h\sigma^2}}\Phibf}_\frob + \stocho(1), \quad
    \norm{\tfrac{1}{\sqrt{h^2\sigma^2}}\Psibf\xbf_0}_2 = \norm{\tfrac{1}{\sqrt{h^2\sigma^2}}\Psibf}_\frob + \stocho(1).
  \end{equation*}
\end{restatable}
\Cref{lemma:norm_concentration,lemma:norm_concentration_x} are based on the \emph{Hanson--Wright inequality} \citep[Theorem 6.3.2]{Vershynin2018}, a concentration inequality for order-$2$ Gaussian chaos.
For example, $\left\|\frac{1}{\sqrt{h\sigma^2}}\Phibf\xbf\right\|_2^2$ can be decomposed into a sum of order-$2$ Gaussian chaos, which is bounded with the Hanson--Wright inequality with high probability.
By combining \cref{lemma:norm_concentration,lemma:norm_concentration_x} with the standard matrix calculus, we can express normalizers $\norm{\Phibf\xbf'}_2^{-1}$ and $\norm{\Wbf\Phibf\xbf}_2^{-1}$ in $\Hbf$ into simpler forms,
and obtain a concise expression of $\Hbf$ consequently.
\begin{restatable}{lemma}{thmevalexpt}
  \hyperlink{proof:evaluation_expectation}{(proof $\blacktriangledown$)}
  \label{lemma:evaluation_expectation}
  Let $\Psibf \defeq \Wbf\Phibf$.
  Assume that $\norm{\Phibf}_\frob$ and $\norm{\Psibf}_\frob$ are bounded away from zero.
  Under \cref{assumption:w_symmetric,assumption:initialization,assumption:standard_mvn,assumption:limit}, $\Hbf$ can be expressed as follows:
  \begin{equation*}
    \Hbf = \frac{\tilde\Phibf\tilde\Psibf^\top - 2\tilde\Psibf\tilde\Phibf^\top\tilde\Psibf\tilde\Psibf^\top - \tr(\tilde\Phibf^\top\tilde\Psibf)\tilde\Psibf\tilde\Psibf^\top}{1+\sigma^2} + \stocho(1),
  \end{equation*}
  where $\tilde\Phibf \defeq \Phibf/\norm{\Phibf}_\frob$ and $\tilde\Psibf \defeq \Psibf/\norm{\Psibf}_\frob$.
\end{restatable}
Hence, with the reparametrization $\Fbf \defeq \Phibf\Phibf^\top$, we drop the asymptotically vanishing term and replace $\Hbf$ with the following $\hat\Hbf$:
\begin{equation*}
  \hat\Hbf = \frac{1}{1+\sigma^2} \left( \frac{\Fbf\Wbf}{N_\Phi N_\Psi} - \frac{2\Wbf\Fbf\Wbf\Fbf\Wbf}{N_\Phi N_\Psi^3} - \frac{N_\times \Wbf\Fbf\Wbf}{N_\Psi^2} \right),
\end{equation*}
where we define $N_\Phi \defeq \norm{\Phibf}_\frob$, $N_\Psi \defeq \norm{\Psibf}_\frob$, and $N_\times \defeq \tr(\Phibf^\top\Psibf)/N_\Phi N_\Psi$.

\begin{remark}[$h$ in asymptotic analysis]
  To make asymptotic terms $\stocho(1)$ in \cref{lemma:norm_concentration,lemma:norm_concentration_x} vanishing, we require $h=\Omega(\exp(\rho t))$, which can be seen by, for example, \cref{equation:proof:hanson_wright} in the appendix.
  By discretizing the continuous dynamics, the continuous time $t$ depends on the discrete time $\bar t$ (i.e., the number of total updates) and step size $\gamma$ via $t=\gamma\bar t$.
  In our simulation in \cref{section:experiments}, we use $(\bar t,\gamma,\rho)=(\num{3000},\num{0.05},\num{0.005})$, leading to $\rho t\lesssim\num{1}$.
  Under these choices, the representation dimension $h$ still stays reasonably small, though we admit this requirement as a limitation of our analysis.
\end{remark}

\begin{remark}[isometric representation]
  According to \cref{lemma:norm_concentration}, the learned representation $\Phibf\xbf$ asymptotically has the same norm regardless of the input $\xbf$.
  This asymptotically isometric behavior is due to the linear encoder, and an expected consequence of our model because we assume the isotropic input covariance (\cref{assumption:standard_mvn}) and focus on study of the representational collapse solely.
  To investigate the non-isometry of representation learning, we need to move on to the non-linear encoder, which is beyond the scope of this article.
\end{remark}

\section{Analysis of non-contrastive dynamics}
\label{section:analysis}

\begin{figure*}[t]
  \includegraphics[width=0.49\textwidth]{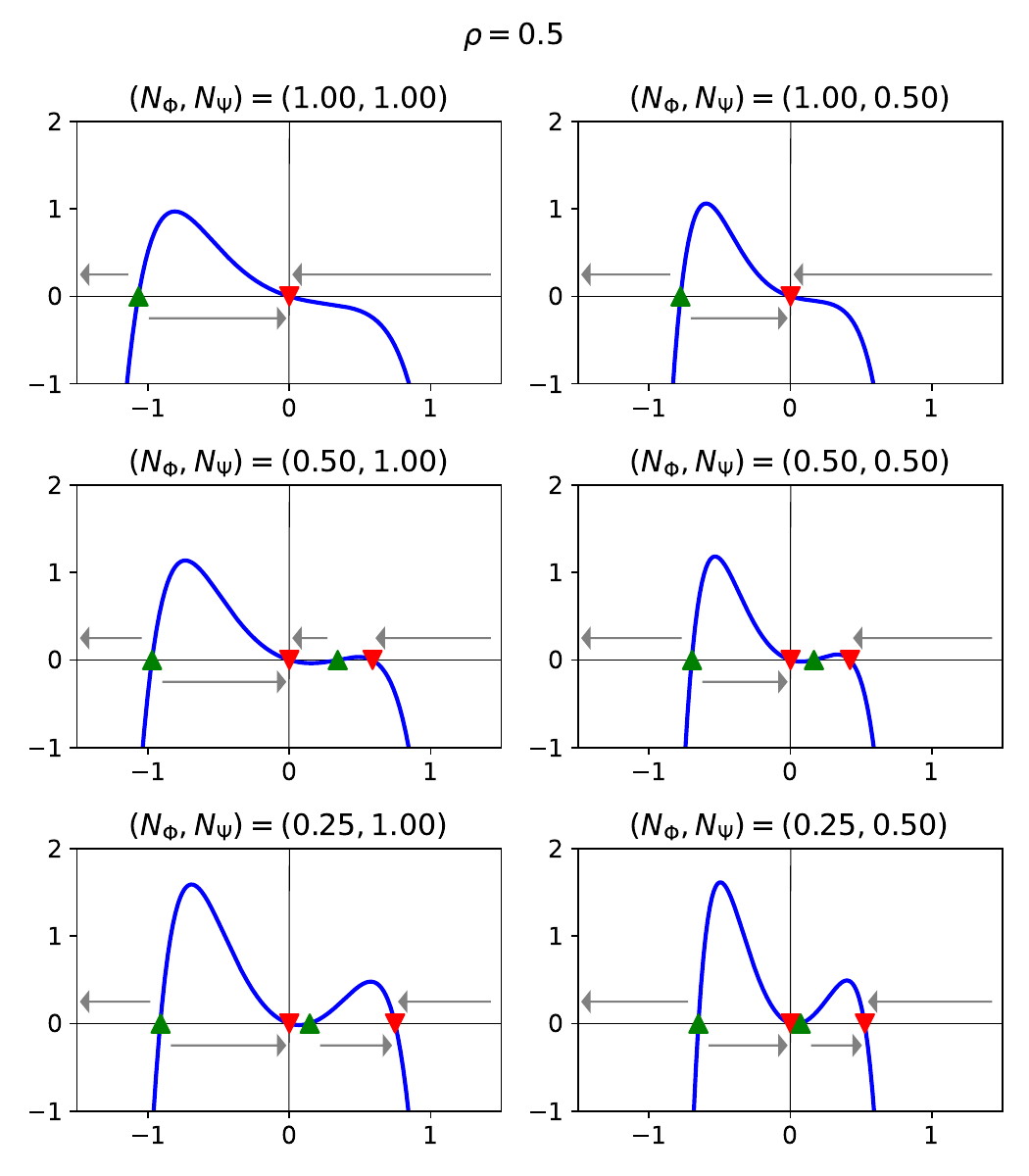}
  \includegraphics[width=0.49\textwidth]{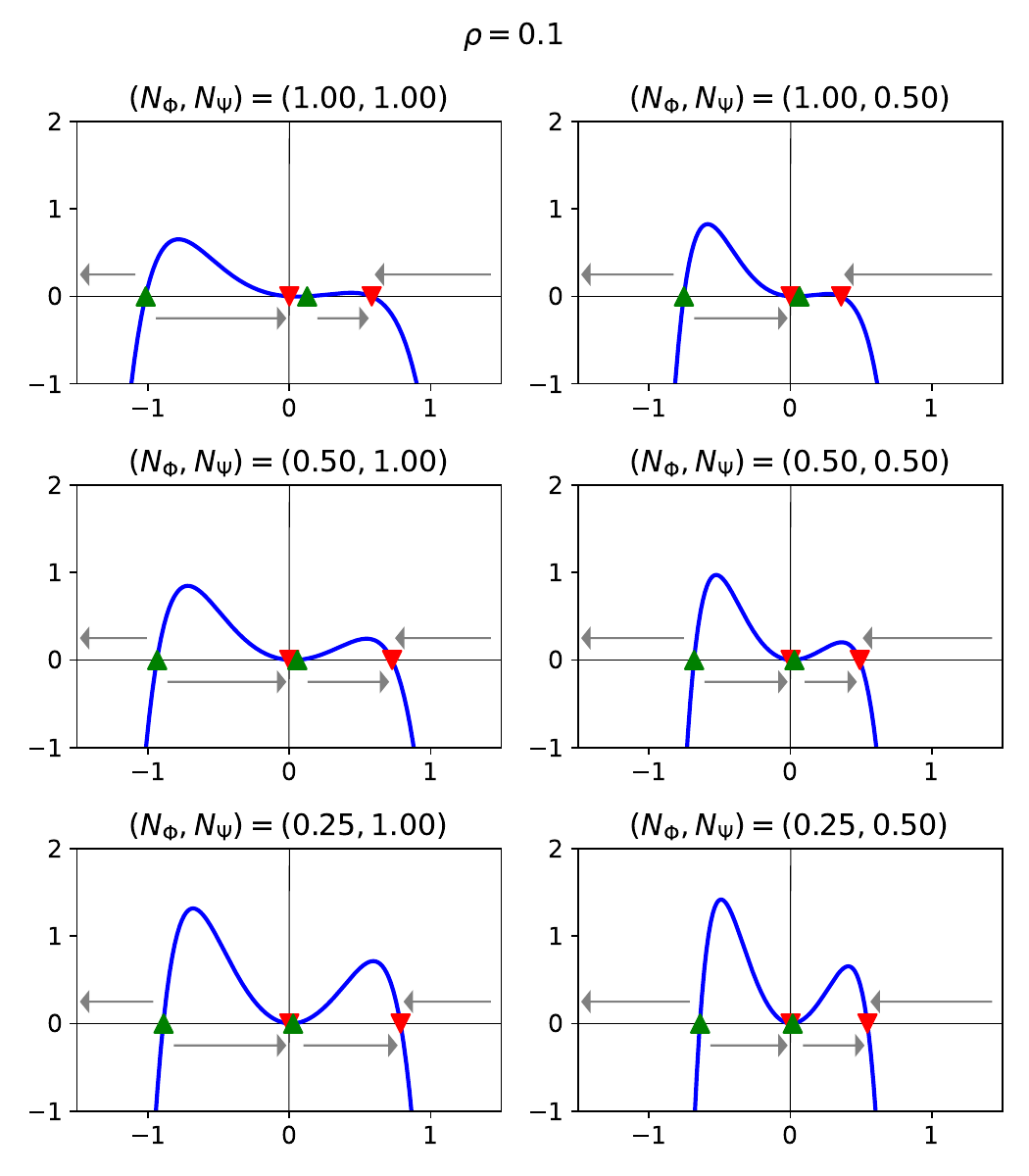}
  \caption{
    Numerical illustrations of the dynamics \cref{equation:dynamics_invariant_parabola} with different values of $(\rho,N_\Phi,N_\Psi)$, where vertical and horizontal axes denote $w_j$ and $\dot{w_j}$, respectively, where $\rho$ is the weight decay parameter, $N_\Phi$ and $N_\Psi$ are the norms of $\Phibf$ and $\Psibf$, respectively.
    The left two columns are illustrated for $\rho=0.5$, while right two columns for $\rho=0.1$.
    Red \textcolor{red!90!black}{$\blacktriangledown$} and green \textcolor{green!50!black}{$\blacktriangle$} indicate stable (namely, $\dot{w_j} < 0$) and unstable equilibrium (namely, $\dot{w_j} > 0$) points, respectively \citep{Hirsch2012}.
    For other parameters, we chose $N_\times = 1$ and $\sigma^2 = 0.1$ for illustration.
  }
  \label{figure:stability}
\end{figure*}

This section aims to analyze the dynamics \eqref{equation:dynamics_matrix_expected} to see when the dynamics has a stable equilibrium.

\subsection{Eigendecomposition of dynamics}
\label{section:eigenvalue_dynamics}
To analyze the stability of the dynamics~\eqref{equation:dynamics_matrix_expected}, we disentangle it into the eigenvalues.
We first show the condition where the eigenspaces of $\Wbf$ and $\Fbf$ align with each other.
Note that two commuting matrices can be simultaneously diagonalized.
\begin{restatable}{proposition}{thmcommutator}
  \hyperlink{proof:commutator_dynamics}{(proof $\blacktriangledown$)}
  \label{proposition:commutator_dynamics}
  Suppose $\Wbf$ is non-singular.
  Under the dynamics~\eqref{equation:dynamics_matrix_expected} with $\Hbf = \hat\Hbf$, the commutator $\Lbf(t) \defeq [\Fbf, \Wbf] \defeq \Fbf\Wbf - \Wbf\Fbf$ satisfies
  $\diff{\vec{\Lbf(t)}}{t} = -\Kbf(t)\vec{\Lbf(t)}$, where
  \begin{equation*}
    \begin{aligned}
      \Kbf(t) &\defeq 2\frac{\Wbf \oplus \Wbf\Fbf\Wbf + \Wbf^2(\Fbf\Wbf \oplus \Ibf_d)}{(1+\sigma^2) N_\Phi N_\Psi^3}\\&\phantom{\defeq} + \frac{(\Wbf^{-1}) \oplus \Fbf - (\Wbf - N_\times \Wbf^2) \oplus \Ibf_d}{(1+\sigma^2) N_\Phi N_\Psi} + 3\rho\Ibf_d,
    \end{aligned}
  \end{equation*}
  and $\Abf \oplus \Bbf \defeq \Abf \otimes \Bbf + \Bbf \otimes \Abf$ denotes the sum of the two Kronecker products.

  If $\inf_{t\ge 0} \eigmin{\Kbf(t)} \ge \lambda_0 > 0$ for some $\lambda_0 > 0$, then $\norm{\Lbf(t)}_\frob \to 0$ as $t \to \infty$.
\end{restatable}
The derivation of the dynamics of~$\Lbf(t)$ follows from the standard matrix calculus.
After deriving the dynamics, we leverage \citet[Lemma~2]{Tian2021ICML} to establish~$\|\Lbf(t)\|_\frob\to0$.
\Cref{proposition:commutator_dynamics} is a variant of \cite[Theorem 3]{Tian2021ICML} for the dynamics~\eqref{equation:dynamics_matrix_expected}.
Consequently, we see that 
$\Wbf$ and $\Fbf$ are simultaneously diagonalizable at the equilibrium $\norm{\Lbf(t)}_\frob = \norm{[\Fbf, \Wbf]}_\frob = 0$.
We then approximately deal with the dynamics~\eqref{equation:dynamics_matrix_expected}.
\begin{assumption}[Always commutative]
  \label{assumption:equilibrium_commutator}
  $\norm{[\Fbf, \Wbf]}_\frob \equiv 0$ for $\forall t \ge 0$.
\end{assumption}
We test the assumption in \cref{section:experiments}, where we see that the commutator can be regarded as being nearly zero.
Let $\Ubf$ be the common eigenvectors of $\Fbf$ and $\Wbf$,
then $\Wbf = \Ubf\Lambdabf_W\Ubf^\top$ and $\Fbf = \Ubf\Lambdabf_F\Ubf^\top$,
where $\Lambdabf_W = \diag[w_1, w_2, \dots, w_h]$ and $\Lambdabf_F = \diag[f_1, f_2, \dots, f_h]$.
By extending the discussion of \cite[Appendix B.1]{Tian2021ICML}, we can show that $\Ubf$ would not change over time.
\begin{restatable}{proposition}{thmeigenspace}
  \hyperlink{proof:stable_eigenspace}{(proof $\blacktriangledown$)}
  \label{proposition:stable_eigenspace}
  Suppose $\Wbf$ is non-singular.
  Under the dynamics~\eqref{equation:dynamics_matrix_expected} with $\Hbf = \hat\Hbf$, we have $\dot\Ubf = \Obf$.
\end{restatable}

With \cref{assumption:equilibrium_commutator} and \cref{proposition:stable_eigenspace},
we decompose \eqref{equation:dynamics_matrix_expected} with $\Hbf = \hat\Hbf$ into the eigenvalues.
\begin{equation}
  \label{equation:dynamics_eigenvalue}
  \begin{cases}
    \dot w_j &= -D_j - \rho w_j, \\
    \dot f_j &= -2D_jw_j - 2\rho f_j, \\
  \end{cases}
\end{equation}
where
\begin{equation*}
  D_j = \frac{1}{1+\sigma^2}\left[\frac{2}{N_\Phi N_\Psi^3}f_j^2w_j^2 + \frac{N_\times}{N_\Psi^2} f_jw_j - \frac{1}{N_\Phi N_\Psi} f_j\right].
\end{equation*}
The eigenvalue dynamics~\eqref{equation:dynamics_eigenvalue} is far more interpretable than the matrix dynamics~\eqref{equation:dynamics_matrix_expected}.
Subsequently, we investigate when learned representations collapse through the stability analysis of the eigenvalue dynamics.
By reducing the matrix dynamics to eigenvalue dynamics, we can leverage common tools of stability analysis in the field of complex systems.

\subsection{Cosine-loss dynamics dynamically yields a stable and non-collapsed equilibrium}
\label{section:regimes}
We are interested in how the eigenvalue avoids collapse with feature normalization.
For this purpose, we investigate the equilibrium points of the eigenvalue dynamics~\eqref{equation:dynamics_eigenvalue}.

\begin{figure*}[t]
  \centering
  \input{figure/regime}
  \caption{
    Schema of Collapse, Acute, and Stable regimes of the cosine-loss eigenvalue dynamics~\cref{equation:dynamics_invariant_parabola}.
    Red \textcolor{red!90!black}{$\blacktriangledown$} and green \textcolor{green!50!black}{$\blacktriangle$} indicate stable (namely, $\dot{w_j} < 0$) and unstable equilibrium (namely, $\dot{w_j} > 0$) points, respectively.
    The black $\blacklozenge$ denotes the saddle point.
    \colorbox{red!20}{\strut Red}, \colorbox{black!20}{\strut gray}, and {\strut blue} backgrounds indicate ranges where the eigenvalue will diverge to $-\infty$, collapse to $0$, and converge to the stable equilibrium, respectively.
    As $N_\Phi$ (norm of $\Phibf$) and $N_\Psi$ (norm of $\Psibf$) become smaller, the dynamics bifurcates in the direction <<Collapse $\to$ Acute $\to$ Stable>>,
    and as $N_\Phi$ and $N_\Psi$ become larger, the dynamics bifurcates in the opposite direction <<Stable $\to$ Acute $\to$ Collapse>>.
  }
  \label{figure:regime}
\end{figure*}

\begin{figure*}[t]
  \centering
  \input{figure/l2}
  \caption{
    Schema of three eigenvalue dynamics in the L2 loss case.
    Each figure illustrates the eigenvalue corresponding \emph{fixed} weight decay $\rho$.
    The meaning of ($\blacktriangle, \blacktriangledown, \blacklozenge$) and background colors can be found in the caption of \cref{figure:regime}.
    The figure borrows the illustration of \cite[Fig.~4]{Tian2021ICML}.
    Importantly, $w_j$ cannot avoid collapse when a given weight decay is too strong so that we are in Case (c), which does not align the practice.
  }
  \label{figure:l2_dynamics}
\end{figure*}

\paragraph*{Invariant parabola.}
By simple algebra, $\dot{f_j} - 2w_j\dot{w_j} = -2\rho(f_j - w_j^2)$.
Noting that $\diff{}{t}(f_j - w_j^2) = \dot{f_j} - 2w_j\dot{w_j}$ and integrating both ends, we encounter the following relation:
\begin{equation}
  \label{equation:asymptotic_parabola}
  f_j(t) = w_j^2(t) + c_j \exp(-2\rho t),
\end{equation}
where $c_j \defeq f_j(0) - w_j^2(0)$ is the initial condition.
\Cref{equation:asymptotic_parabola} elucidates that the dynamics of $(w_j(t), f_j(t))$ asymptotically converges to the parabola $f_j(t) = w_j^2(t)$ as $t \to \infty$
when regularization $\rho > 0$ exists.
The information of initialization $c_j$ shall be forgotten.
Stronger regularization yields faster convergence to the parabola.
We reasonably expect that this exponential convergence is much faster than the drifts of $w_j$ and $f_j$ so that they are constrained to the parabola quickly.

\paragraph*{Dynamics on invariant parabola.}
We now focus on the dynamics on the invariant parabola.
Substituting $f_j(t) = w_j^2(t)$ into $w_j$-dynamics in \cref{equation:dynamics_eigenvalue} yields the following dynamics:
\begin{equation}
  \label{equation:dynamics_invariant_parabola}
  \begin{aligned}
    &\text{(Cos-loss dynamics)} \\
    &\dot{w_j} = -\frac{2}{(1+\sigma^2)N_\Phi N_\Psi^3}w_j^6 - \frac{N_\times}{(1+\sigma^2)N_\Psi^2}w_j^3 + \frac{1}{(1+\sigma^2)N_\Phi N_\Psi}w_j^2 - \rho w_j.
  \end{aligned}
\end{equation}

We illustrate the dynamics~\eqref{equation:dynamics_invariant_parabola} with different parameter values in \cref{figure:stability}.
This dynamics always has $w_j = 0$ as an equilibrium point, and the number of equilibrium points varies between two and four.
Notably, \cref{equation:dynamics_invariant_parabola} is a \emph{sixth-order} non-linear ODE (in $w_j$),
whereas the L2 loss dynamics \citep[Eq.~(16)]{Tian2021ICML} induces a \emph{third-order} non-linear eigenvalue dynamics, as we will show in \cref{section:comparison}.
From \cref{figure:stability}, we can classify into three regimes (refer to \cref{figure:regime} together).
The detailed derivation of these regimes can be found in \cref{section:regime_shift}.

\paragraph{\raisebox{-1pt}{\Sadey} Collapse regime.}
When $\rho$, $N_\Phi$, and $N_\Psi$ are large altogether, the dynamics only has two equilibrium points.
For example, see the plots in \cref{figure:stability} with $(\rho,N_\Phi,N_\Psi) \in \set{(0.5,1.0,1.0), (0.5,1.0,0.5)}$.
In this regime, $w_j=0$ is the only stable equilibrium, causing the collapsed dynamics.
This regime is brittle because the stable equilibrium $w_j=0$ blows up the normalizers $N_\Phi^{-1}$ and $N_\Psi^{-1}$ in the dynamics~\eqref{equation:dynamics_invariant_parabola}.
As $w_j$ shrinks, the values $N_\Phi$ and $N_\Psi$ shrink together, too, which brings the dynamics into the next two regimes.

\paragraph{\raisebox{-1pt}{\Neutrey} Acute regime.}
When $N_\Phi$ and $N_\Psi$ become smaller than those in Collapse, two new equilibrium points emerge and the number of equilibrium points is four in total.
For example, see the plots in \cref{figure:stability} with $(\rho,N_\Phi,N_\Psi) \in \set{(0.5,0.5,0.5), (0.1,1.0,1.0)}$.
Let $w_\blacktriangle^{(-)}$, $w_\blacktriangledown^{(0)} (= 0)$, $w_\blacktriangle^{(+)}$, and $w_\blacktriangledown^{(+)}$ denote the equilibrium points from smaller to larger ones, respectively, namely, $w_\blacktriangle^{(-)} < w_\blacktriangledown^{(0)} = 0 < w_\blacktriangle^{(+)} < w_\blacktriangledown^{(+)}$ (see \cref{figure:regime}).
Note that $w_j = w_\blacktriangle^{(-)}, w_\blacktriangle^{(+)}$ are unstable and $w_j = w_\blacktriangledown^{(0)}, w_\blacktriangledown^{(+)}$ are stable \citep{Hirsch2012}.
In this regime, the eigenvalue initialized larger than $w_\blacktriangle^{(+)}$ converge to non-degenerate point $w_\blacktriangledown^{(+)}$.
However, the eigenvalue degenerates to $w_\blacktriangledown^{(0)}$ if initialization is in the range $[w_\blacktriangle^{(-)}, w_\blacktriangle^{(+)}]$ (close to zero),
and diverges if initialization has large negative value $< w_\blacktriangle^{(-)}$.
If the eigenvalue degenerates, the values $N_\Phi$ and $N_\Psi$ further shrink and then the regime enters the final one;
if the eigenvalue diverges, $N_\Phi$ and $N_\Psi$ inflate and the regime goes back to the previous Collapse.

\paragraph{\raisebox{-1pt}{\Laughey} Stable regime.}
When $N_\Phi$ and $N_\Psi$ are further smaller than those in Acute, the middle two equilibrium points $w_\blacktriangledown^{(0)}$ and $w_\blacktriangle^{(+)}$ approach and form a saddle point.
For example, see the plots in \cref{figure:stability} with $(\rho,N_\Phi,N_\Psi) \in \set{(0.5,0.25,0.5), (0.1,0.25,0.5)}$.
Denote this saddle point by $w_\blacklozenge$.
The dynamics has a unstable equilibrium $w_\blacktriangle^{(-)}$, a saddle point $w_\blacklozenge$, and a stable equilibrium $w_\blacktriangledown^{(+)}$, from smaller to larger ones.
In this regime, the eigenvalue stably converges to the non-degenerate point $w_j = w_\blacktriangledown^{(+)}$ unless the initialization is smaller than $w_\blacktriangle^{(-)}$.

\begin{remark}
  $w_\blacktriangledown^{(0)} = w_\blacktriangle^{(+)}$ never occurs and neither does the Stable regime because the dynamics diverges as $N_\Phi, N_\Psi \to 0$.
  Nonetheless, this approximately occurs with realistic parameters such as $(\rho,N_\Phi,N_\Psi)=(0.1,0.25,0.5)$.
\end{remark}

\paragraph*{Three regimes prevent collapse.}
To wrap up, we argue that the dynamics \eqref{equation:dynamics_invariant_parabola} eventually converges to the stable equilibrium $w_\blacktriangledown^{(+)}$ that exists in Acute and Stable regimes, even if the initial regime is Collapse.
We illustrate the three regimes and this concept in \cref{figure:regime}.
As we see in the numerical experiments (\cref{section:experiments}), the parameter initialization (\cref{assumption:initialization}) hardly makes the initial eigenvalue smaller than $w_\blacktriangle^{(-)}$:
indeed, we simulated the initial eigenvalue distributions in \cref{figure:eigeninit}, which indicates that the eigenvalues are sufficiently larger than $w_\blacktriangle^{(-)}$.
Therefore, the learning dynamics has stable equilibriums and successfully stabilizes.
\begin{figure}[t]
  \centering
  \includegraphics[width=0.48\textwidth]{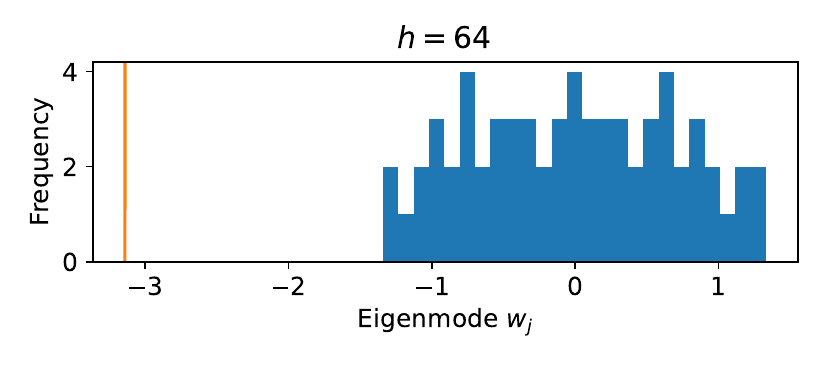}
  \includegraphics[width=0.48\textwidth]{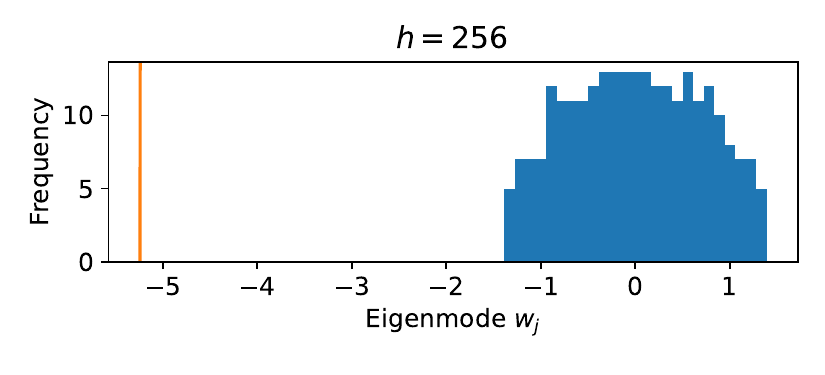}
  \caption{
    Numerical simulation of eigenvalue distributions of $\Wbf$.
    In each figure, we generate $\Wbf$ and $\Phibf$ by the initialization of \cref{assumption:initialization}, and illustrate the histogram of eigenvalues of $\Wbf$.
    The vertical line indicates the value of $w_\blacktriangle^{(-)}$, the negative unstable equilibrium point of $w_j$-dynamics \eqref{equation:dynamics_invariant_parabola},
    computed by the binary search and numerical root finding.
    For parameters, we chose $\rho=0.05$, $\sigma^2=1.0$, $d=2048$, and $h \in \set{64, 256}$.
  }
  \label{figure:eigeninit}
\end{figure}

Importantly, \emph{the cosine loss dynamics~\eqref{equation:dynamics_invariant_parabola} can stabilize and would not collapse to zero regardless of the regularization strength $\rho$},
which is in stark contrast to the L2 loss dynamics, as detailed in \cref{section:comparison}.
This observation tells us the importance of feature normalization to prevent representation collapse in non-contrastive self-supervised learning.
Note that the eigenvalue $w_j$ can explode to the negative infinity even in Stable regime (if it is initialized smaller than $w_\blacktriangle^{(-)}$), which is an interesting observation from the dynamics analysis beyond a simple observation that having feature normalization would never collapse the solution towards zero.

\subsection{L2-loss dynamics cannot escape collapse with intense weight decay}
\label{section:comparison}

Whereas we focused on the study of the cosine loss dynamics, \cite{Tian2021ICML} and many earlier studies engaged in the L2 loss dynamics, which is simple but does not entail feature normalization.
Here, we compare the cosine and L2 loss dynamics to see how feature normalization plays a crucial role.

Let us review the dynamics of \cite{Tian2021ICML}.
We inherit \cref{assumption:w_symmetric} (symmetric projector), \cref{assumption:standard_mvn} (standard normal input), and \cref{assumption:equilibrium_commutator} ($\Fbf$ and $\Wbf$ are commutative).
Under this setup, \cite{Tian2021ICML} analyzed the non-contrastive dynamics \eqref{equation:dynamics_matrix_expected} with the L2 loss \eqref{equation:l2_loss}, and revealed that the eigenvalues of $\Wbf$ and $\Fbf$ (denoted by $w_j$ and $f_j$, respectively) asymptotically converges to the invariant parabola $f_j(t) = w_j^2(t)$ (see \cref{equation:asymptotic_parabola}),
where the $w_j$-dynamics reads:
\begin{equation}
  \text{(L2-loss dynamics)} \qquad
  \dot{w_j} = w_j^2 \{1 - (1+\sigma^2)w_j\} - \rho w_j.
  \label{equation:eigenvalue_dynamics_l2}
\end{equation}
Compare the L2-loss dynamics~\eqref{equation:eigenvalue_dynamics_l2} (third-order) and the cosine-loss dynamics~\eqref{equation:dynamics_invariant_parabola} (sixth-order).
Even if we leverage the norm constancy $N_\Phi \equiv \|\Phibf\|_\frob$ and $N_\Psi \equiv \|\Psibf\|_\frob$ at the proportional limit, the cosine-loss does not reduce to the L2-loss dynamics just because the cosine loss entails gradients~\eqref{equation:dynamics_matrix_expected} with higher nonlinearity.
Note that we omit the exponential moving average of the online representation in BYOL ($\tau = 1$) and use the same learning rate for the predictor and online nets ($\alpha = 1$) in \cite{Tian2021ICML} for comparison to our dynamics~\eqref{equation:dynamics_invariant_parabola}.

The behaviors of the two dynamics are compared in \cref{figure:regime} (cosine loss) and \cref{figure:l2_dynamics} (L2 loss).
One of the most important differences is that the cosine loss dynamics has the saddle-node bifurcation depending on $N_\Phi$, $N_\Psi$, and $N_\times$,
while the L2 loss dynamics does not have such a bifurcation.
Thus, the L2 loss dynamics \eqref{equation:eigenvalue_dynamics_l2} and its time evolution are solely determined by a given regularization strength $\rho$ (see three plots in \cref{figure:l2_dynamics}).
That being said, \emph{the eigenvalue cannot stably converges but collapses to zero if the L2 loss dynamics is excessively regularized} such that $\rho > \frac{1}{4(1+\sigma^2)}$.
On the contrary, the cosine loss with a strong regularization may initially make the dynamics fall into the Collapse regime, where no meaningful stable equilibrium exists,
but the regime gradually bifurcates to Acute as the eigenvalue (and the norms $N_\Phi$ and $N_\Psi$ accordingly) approaches zero.
Such a bifurcation owes to feature normalization involved in the cosine loss.

\subsection{Numerical experiments}
\label{section:experiments}

\begin{figure*}[t]
  \centering
  \includegraphics[width=0.325\textwidth]{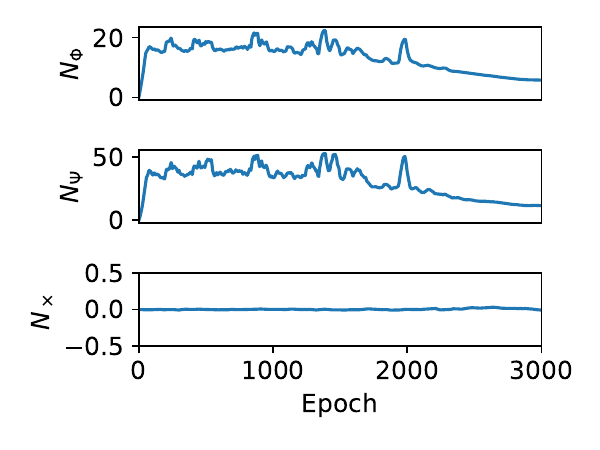} \hfill
  \includegraphics[width=0.325\textwidth]{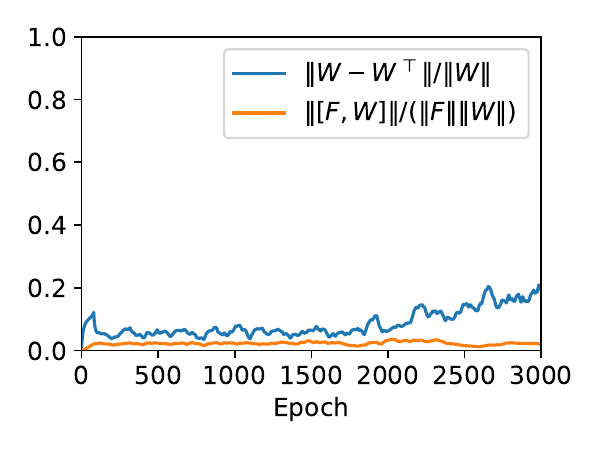} \hfill
  \includegraphics[width=0.325\textwidth]{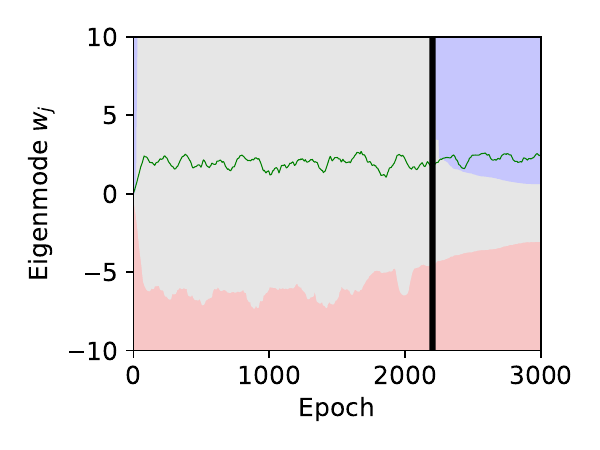}
  \caption{
    Numerical simulation of the SimSiam model.
    \textbf{(Left)} Time evolution of $N_\Phi$ (norm of $\Psibf$), $N_\Psi$ (norm of $\Psibf$), and $N_\times$ ($=\tr(\Phibf^\top\Psibf)/N_\Phi N_\Psi$). We see the gradual shrinkages of $N_\Phi$ and $N_\Psi$.
    \textbf{(Center)} Asymmetry of the projection head $\Wbf$ (measured by the relative error of $\Wbf - \Wbf^\top$) and non-commutativity of $\Fbf$ and $\Wbf$ (measured by the relative error of the commutator $[\Fbf,\Wbf]$). The relative errors stay close to zero during time evolution (cf. \cref{assumption:w_symmetric,assumption:equilibrium_commutator}).
    \textbf{(Right)} The leading eigenvalue of the projection head $w_j$ (green line), with background colors illustrating three intervals where \colorbox{red!20}{\raisebox{0pt}[0.8ex]{$w_j$ diverges}}, \colorbox{black!20}{\raisebox{0pt}[0.8ex]{$w_j$ collapses}}, and \colorbox{blue!20}{\raisebox{0pt}[0.8ex]{$w_j$ stably converges}} at each epoch. The regime boundaries are numerically computed by the binary search and root finding of \eqref{equation:dynamics_invariant_parabola}. Each color corresponds to those in \cref{figure:regime}. The vertical black line indicates the bifurcation from Collapse ($\mathrm{epoch} < 2200$) to Acute ($\mathrm{epoch} > 2200$).
  }
  \label{figure:experiment}
\end{figure*}

We conducted a simple numerical simulation of the SimSiam model using the official implementation available at \url{https://github.com/facebookresearch/simsiam}.
We tested the linear model setup shown in \cref{section:model},
with linear representation net $\Phibf$ and linear projection head $\Wbf$, and the representation dimension was set to $h=64$.
Data are generated from the $512$-dimensional ($d=512$) standard multivariate normal (\cref{assumption:standard_mvn}) and data augmentation follows isotropic Gaussian noise $\Dcal_{\xbf_0}^\mathrm{aug}$, with variance $\sigma^2 = 1.0$.
The learning rate of the momentum SGD was initially set to $0.05$ and scheduled by the cosine annealing.
The regularization strength was set to $\rho=0.005$.
For the other implementation details, we followed the official implementation.

The results are shown in \cref{figure:experiment}.
In the left figure, we illustrate how $N_\Phi$, $N_\Psi$, $N_\times$ drift over the time.
As seen, $N_\Phi$ and $N_\Psi$ gradually shrinks along the time, which theoretically leads to the saddle-node bifurcation.
This bifurcation is empirically observed in \cref{figure:experiment} (Right).
In this figure, we compute the theoretical intervals where $w_j$ diverges, collapses and stably converges by numerically solving the equilibrium equation $\dot{w_j}=0$ with the dynamics \eqref{equation:dynamics_invariant_parabola}.
At each epoch, the background colors corresponds to the theoretical intervals.
The regularization strength $\rho=0.005$ used in this experiment is rather larger than the default SimSiam regularization strength $\rho=10^{-4}$,
which leads to the Collapse regime initially (when $\mathrm{epoch} < 2200$) but gradually bifurcates to the Acute regime (when $\mathrm{epoch} > 2200$).
Thus, we observed how the eigenvalue escapes from the Collapse regime.
Lastly, we quickly confirm the validity of \cref{assumption:w_symmetric,assumption:equilibrium_commutator} by measuring the asymmetry of the projection head $\Wbf$ and commutativity of $\Fbf$ and $\Wbf$ in \cref{figure:experiment} (Center),
which suggests that the assumptions can be said reasonable in general.
More empirical analyses (together with the other eigenvalues; additionally, the simulation with ResNet-18 encoder) can be found in \cref{section:additional_experiments}.

\section{Conclusion and limitations}
\label{section:conclusion}

In this work, we questioned how to describe non-contrastive dynamics without complete collapse.
The existing theory \citep{Tian2021ICML} leverages the simplicity of the L2 loss to analytically derive the dynamics of the two-layer non-contrastive learning, while regularization unexpectedly affects collapse.
This may indicate a drawback of the L2 loss analysis, though their theoretical model is transparent.
Alternatively, we focused on the cosine loss, which involves feature normalization
and derived the corresponding eigenvalue dynamics.
Despite that the dynamics may fall into the Collapse regime for too strong regularization, the shrinkage of the eigenvalues brings the regime into non-collapse ones.
Thus, we witnessed the importance of feature normalization.
Technically, we leveraged the proportional limit, which allows us to focus on concentrated feature norms.
We believe that a similar device may enhance theories of related architectures, including self-supervised learning based on covariance regularization such as Barlow Twins and VICReg.

This work is limited in two ways.
First, we do not answer what non-contrastive dynamics learns.
While downstream performances of contrastive learning have been theoretically analyzed through the lens of the learning theoretic viewpoint \citep{Saunshi2019ICML,Nozawa2021NeurIPS,Wang2022ICLR,Bao2022ICML} and the smoothness of loss landscapes \citep{Liu2023ICML},
we have far less understanding of non-contrastive learning for the time being.
Second, our analysis hinges on dynamics at a fixed time $t$, and we do not solve \eqref{equation:dynamics_matrix_expected}, which is challenging but interesting from the perspective of dynamics.

\section*{Acknowledgments}
HB appreciates Yoshihiro Nagano for providing numerous insights at the initial phase of this research.
A part of the experiments of this research was conducted using Wisteria/Aquarius in the Information Technology Center, The University of Tokyo.

\bibliographystyle{alpha}
\bibliography{reference}

\newpage
\appendix
\onecolumn

\begin{center}
  \fontsize{14.4pt}{20pt}
  \selectfont
  \rule[5pt]{\columnwidth}{1pt} \\
  \textsc{
    Appendix
  } \\
  \rule[0pt]{\columnwidth}{1pt}
\end{center}

\section{Technical lemmas}
\label{section:technical_lemmas}

\subsection{Sub-Weibull distributions}
In this subsection, we give a brief introduction to \emph{sub-Weibull distributions} \citep{Hao2019NeurIPS,Vladimirova2020},
which is a generalization of seminal sub-Gaussian and sub-exponential random variables.
First, we define sub-Weibull distributions.
\begin{definition}[\cite{Hao2019NeurIPS}]
  \label{definition:sub_weibull}
  For $\beta > 0$, we define $X$ as a sub-Weibull random variable with the $\psi_\beta$-norm if it entails a bounded $\psi_\beta$-norm, defined as follows:
  \begin{equation*}
    \norm{X}_{\psi_\beta} \defeq \inf\setcomp{C \in (0, \infty)}{\E[\exp(\abs{X}^\beta / C^\beta)] \le 2}.
  \end{equation*}
\end{definition}
We occasionally call \emph{$\beta$-sub-Weibull} to specify the corresponding $\psi_\beta$-norm explicitly.
Obviously, $\beta = 2$ and $\beta = 1$ recover sub-Gaussian and sub-exponential distributions, respectively.
Among equivalent definitions of sub-Weibull distributions, we often use the following conditions.
\begin{proposition}[\cite{Vladimirova2020}]
  \label{proposition:sub_weibull}
  Let $X$ be a sub-Weibull random variable. Then, the following conditions are equivalent:
  \begin{enumerate}
    \item The tails of $X$ satisfy
    \begin{equation*}
      \exists K_1 > 0 \quad \text{such that} \quad \Prob{|X| \ge \epsilon} \le 2\exp\left(-(\epsilon/K_1)^\beta\right) \quad \text{for all } \epsilon \ge 0.
    \end{equation*}

    \item The moments of $X$ satisfy
    \begin{equation*}
      \exists K_2 > 0 \quad \text{such that} \quad \norm{X}_{L^p} \defeq \{\E|X|^p\}^{1/p} \le K_2p^{1/\beta} \quad \text{for all } p \ge 1.
    \end{equation*}

    \item The moment-generating function (MGF) of $|X|^\beta$ is bounded at some point, namely,
    \begin{equation*}
      \exists K_3 > 0 \quad \text{such that} \quad \E\exp\left((|X|/K_3)^\beta\right) \le 2.
    \end{equation*}
  \end{enumerate}
  The parameters $K_1$, $K_2$, and $K_3$ differ from each other by at most an absolute constant factor.
\end{proposition}
We are interested in sub-Weibull distributions because they admit a nice closure property, as shown below.
\begin{proposition}[\cite{Vladimirova2020}]
  \label{proposition:sub_weibull_closure}
  Let $X$ and $Y$ be $\beta$-sub-Weibull random variables.
  Then, $XY$ is $(\beta/2)$-sub-Weibull with $\norm{XY}_{\psi_{\beta/2}} \le \norm{X}_{\psi_\beta} \norm{Y}_{\psi_\beta}$.
  In addition, $X + Y$ is $\beta$-sub-Weibull with $\norm{X + Y}_{\psi_\beta} \le \norm{X}_{\psi_\beta} + \norm{Y}_{\psi_\beta}$.
\end{proposition}
Note that \cref{proposition:sub_weibull_closure} does not require the independence of two random variables $X$ and $Y$.
Lastly, we show a corresponding concentration inequality for the sum of independent sub-Weibull random variables,
which is a generalization of Hoeffding's and Bernstein's inequalities for sub-Gaussian and sub-exponential random variables, respectively.
\begin{proposition}[\cite{Hao2019NeurIPS}]
  \label{proposition:sub_weibull_concentration}
  Let $X_1, \dots, X_N$ be independent $\beta$-sub-Weibull random variables with $\norm{X_i}_{\psi_\beta} \le K$ for each $i \in [N]$.
  Then, there exists an absolute constant $C > 0$ only depending on $\beta$ such that for any $\delta \in (0, e^{-2})$,
  \begin{equation*}
    \abs{\sum_{i=1}^NX_i - \E\left[\sum_{i=1}^NX_i\right]} \le CK\left(\sqrt{N\log\frac{1}{\delta}} + \left(\log\frac{1}{\delta}\right)^{1/\beta}\right),
  \end{equation*}
  with probability at least $1-\delta$.
\end{proposition}
For the proofs of these propositions, please refer to the corresponding references.

We additionally provide technical lemmas for random matrices whose element is sub-Weibull.
\begin{lemma}
  \label{lemma:frobenius_norm_bound}
  Let $\Gbf \in \Rbb^{h \times d}$ be a random matrix with each element being $\beta$-sub-Weibull such that $\norm{G_{ij}}_{\psi_\beta} = \Ocal(K(d,h))$ for some $\beta > 0$ and any $(i, j) \in [h] \times [d]$,
  where $K(d,h)$ may depend on $d$ and $h$.
  Then, $\frac{1}{d^2h^2}\norm{\Gbf^\top\Gbf}_\frob^2 = \stochO(K(d,h)/h)$.
\end{lemma}

\begin{proof}[Proof of \cref{lemma:frobenius_norm_bound}]
  Let $\Gbf_i \in \Rbb^h$ denote the $i$-th column vector of the matrix $\Gbf$.
  We have the decomposition $\norm{\Gbf^\top\Gbf}_\frob^2 = \sum_{i,j=1}^d\inpr{\Gbf_i}{\Gbf_j}^2$.
  Let us focus on each $\inpr{\Gbf_i}{\Gbf_j}$ for fixed $i$ and $j$ first.
  We can decompose into $\inpr{\Gbf_i}{\Gbf_j} = \sum_{k=1}^hG_{ik}G_{jk}$,
  which is the sum of $(\beta/2)$-sub-Weibull random variable $G_{ik}G_{jk}$ with $\norm{G_{ik}G_{jk}}_{\psi_{\beta/2}} = \Ocal(K(d,h))$ (cf. \cref{proposition:sub_weibull_closure}).
  By using the closure property under addition (\cref{proposition:sub_weibull_closure}),
  the sum $\inpr{\Gbf_i}{\Gbf_j}$ is $(\beta/2)$-sub-Weibull again,
  with $\norm{\inpr{\Gbf_i}{\Gbf_j}}_{\psi_{\beta/2}} = \Ocal(hK(d,h))$.

  Now, we move back to evaluation of $\norm{\Gbf^\top\Gbf}_\frob^2 = \sum_{i,j=1}^d\inpr{\Gbf_i}{\Gbf_j}^2$.
  By using the closure property under multiplication (\cref{proposition:sub_weibull_closure}),
  $\inpr{\Gbf_i}{\Gbf_j}^2$ is $(\beta/4)$-sub-Weibull with $\norm{\inpr{\Gbf_i}{\Gbf_j}^2}_{\psi_{\beta/4}} = \Ocal(hK(d,h))$.
  Then, the closure property under addition implies that $\sum_{i,j=1}^d\inpr{\Gbf_i}{\Gbf_j}^2$ is $(\beta/4)$-sub-Weibull,
  with $\norm{\sum_{i,j=1}^d\inpr{\Gbf_i}{\Gbf_j}^2}_{\psi_{\beta/4}} = \Ocal(d^2hK(d,h))$.
  Hence, by using the sub-Weibull tails in \cref{proposition:sub_weibull},
  \begin{equation*}
    \norm{\Gbf^\top\Gbf}_\frob^2
    = \sum_{i,j=1}^d\inpr{\Gbf_i}{\Gbf_j}^2
    = \stochO(d^2hK(d,h)),
  \end{equation*}
  from which we deduce that $\frac{1}{d^2h^2}\norm{\Gbf^\top\Gbf}_\frob^2 = \stochO(K(d,h)/h)$.
\end{proof}

\begin{lemma}
  \label{lemma:spectral_norm_bound}
  Let $\Gbf \in \Rbb^{h \times d}$ be a random matrix with each element being $\beta$-sub-Weibull such that $\norm{G_{ij}}_{\psi_\beta} = \Ocal(K(d,h))$ for some $\beta > 0$ and any $i,j \in [d]$,
  where $K(d,h)$ may depend on $d$ and $h$.
  Then, $\norm{\Gbf} = \stochO((d^{1/\beta}+h^{1/\beta})K(d,h))$.
\end{lemma}

\begin{proof}[Proof of \cref{lemma:spectral_norm_bound}]
  The proof is akin to \citet[Theorem 4.4.5]{Vershynin2018}, which is a spectral norm deviation for sub-Gaussian random matrices.
  We leverage the \emph{$\epsilon$-net argument}:
  Using \citet[Corollary 4.2.13]{Vershynin2018}, we can find $\epsilon$-nets $\Mcal_d$ of $\Sbb^{d-1}$ with $\abs{\Mcal_d} \le 9^d$
  and $\Mcal_h$ of $\Sbb^{h-1}$ with $\abs{\Mcal_h} \le 9^h$,
  and $\norm{\Gbf} \le 2\max_{\xbf \in \Mcal_d, \ybf \in \Mcal_h}\inpr{\Gbf\xbf}{\ybf}$.
  Hence, it is sufficient to control the quadratic form $\inpr{\Gbf\xbf}{\ybf}$ for fixed $(\xbf, \ybf) \in \Mcal_d \times \Mcal_h$.

  The quadratic form $\inpr{\Gbf\xbf}{\ybf} = \sum_{i=1}^d\sum_{j=1}^hG_{ij}x_iy_j$ is the sum of $\beta$-sub-Weibull random variables.
  By the closure property (\cref{proposition:sub_weibull_closure}),
  \begin{equation*}
    \norm{\inpr{\Gbf\xbf}{\ybf}}_{\psi_\beta}^2
    \le \sum_{i,j}\norm{G_{ij}x_iy_j}_{\psi_\beta}^2
    \le \Ocal(K(d,h)) \cdot \left(\sum_{i=1}^dx_i^2\right) \left(\sum_{j=1}^hy_j^2\right)
    = \Ocal(K(d,h)).
  \end{equation*}
  Thus, sub-Weibull tails (\cref{proposition:sub_weibull}) imply
  $\Prob{\inpr{\Gbf\xbf}{\ybf} \ge u} \le 2\exp(-(u/K_1)^\beta)$ with $K_1 = \Ocal(K(d,h))$.
  The union bound yields
  \begin{equation*}
    \Prob{\max_{\xbf \in \Mcal_d, \ybf \in \Mcal_h} \inpr{\Gbf\xbf}{\ybf} \ge u}
    \le 9^{d + h} \cdot 2\exp(-(u/K_1)^\beta)
    \le 2\exp(-\delta^\beta),
  \end{equation*}
  where the last inequality is a consequence of the choice $u = CK_1(d^{1/\beta}+h^{1/\beta}+\delta)$ with a sufficiently large absolute constant $C$.
  Hence, $\Prob{\norm{\Gbf} \ge 2u} \le 2\exp(-\delta^\beta)$ holds,
  namely, $\norm{\Gbf} = 2C(d^{1/\beta} + h^{1/\beta} + \delta) \cdot \Ocal(K(d,h))$ holds with probability at least $1 - 2\exp(-\delta^\beta)$.
  This completes the proof.
\end{proof}

\subsection{Integral inequality}
In this subsection, we briefly introduce the Gr\"{o}nwall--Bellman inequality \citep{Bellman1943,Gidel2019NeurIPS}
to solve functional inequalities represented by integrals.
In subsequent analyses, we heavily use it to control the norm of certain random matrices during time evolution.

\begin{theorem}[Gr\"{o}nwall--Bellman inequality]
  \label{theorem:gronwall_bellman}
  Let $\beta$ be a non-negative function and $\alpha$ a non-decreasing function.
  Let $u$ be a function defined on an interval $\Ical = [0, \infty)$ such that
  \begin{equation*}
    u(t) \le \alpha(t) + \int_0^t \beta(s)u(s)\rd{s}, \quad \forall t \in \Ical.
  \end{equation*}
  Then, we have
  \begin{equation*}
    u(t) \le \alpha(t) \exp\left(\int_0^t\beta(s)\rd{s}\right), \quad \forall t \in \Ical.
  \end{equation*}
\end{theorem}

\subsection{Helper lemmas}

\begin{lemma}
  \label{lemma:spectral_norm_init}
  Under the initialization of \cref{assumption:initialization}, we have the following results:
  \begin{enumerate}
    \item \label{enum:spectral_norm_init:phi2}
    $\frac{1}{h}\norm{\Phibf^\top\Phibf(0)} = \stocho(1)$.

    \item \label{enum:spectral_norm_init:phiw2}
    $\frac{1}{h^2}\norm{\Phibf^\top\Wbf^\top\Wbf\Phibf(0)} = \stocho(1)$.
  \end{enumerate}
\end{lemma}

\begin{proof}[Proof of \cref{lemma:spectral_norm_init}]
  To prove \ref{enum:spectral_norm_init:phi2}, 
  we note that each element of the random matrix $\sqrt{d}\Phibf(0)$ is sub-Gaussian (namely, $2$-sub-Weibull)
  with the $\psi_2$-norm being $\Ocal(1)$,
  by the assumption on the parameter initialization (\cref{assumption:initialization}).
  Then, \cref{lemma:spectral_norm_bound} implies $d\norm{\Phibf^\top\Phibf(0)} = \norm{\sqrt{d}\Phibf(0)}^2 = \stochO(d)$.
  Finally, we have $\frac{1}{h}\norm{\Phibf^\top\Phibf(0)} = \stochO(1/h) = \stocho(1)$.

  The identity \ref{enum:spectral_norm_init:phiw2} follows similarly.
  The $(i,j)$-th element of the random matrix $\sqrt{dh}\Wbf\Phibf(0)$ can be expressed as $\inpr{\wbf_i}{\Phibf_j}$,
  where $\wbf_i$ is the $i$-th row vector of $\sqrt{h}\Wbf(0)$ and $\Phibf_j$ is the $j$-th column vector of $\sqrt{d}\Phibf(0)$.
  Both $\wbf_i$ and $\Phibf_j$ are $h$-dimensional vectors with each element being standard normal.
  Hence, $\inpr{\wbf_i}{\Phibf_j}$ is the sum of $h$ sub-exponential random variables, being sub-exponential with $\norm{\inpr{\wbf_i}{\Phibf_j}}_{\psi_1} = \Ocal(h)$ (by using \cref{proposition:sub_weibull_closure}).
  This indicates that each element of $\sqrt{dh}\Wbf\Phibf(0)$ is sub-exponential (namely, $1$-sub-Weibull).
  Then, \cref{lemma:spectral_norm_bound} implies $dh\norm{\Phibf^\top\Wbf^\top\Wbf\Phibf(0)} = \norm{\sqrt{dh}\Wbf\Phibf(0)}^2 = \stochO(d^2)$.
  Finally, we have $\frac{1}{h^2}\norm{\Phibf^\top\Wbf^\top\Wbf\Phibf(0)} = \stochO(1/h^2) = \stocho(1)$.
\end{proof}

\begin{lemma}
  \label{lemma:frobenius_norm_init}
  Under the initialization of \cref{assumption:initialization}, we have the following results:
  \begin{enumerate}
    \item \label{enum:frobenius_norm_init:phi}
    $\frac{1}{h^2}\norm{\Phibf^\top\Phibf(0)}_\frob^2 = \stocho(1)$.

    \item \label{enum:frobenius_norm_init:phiw}
    $\frac{1}{h^2}\norm{\Phibf^\top\Wbf^\top\Wbf\Phibf(0)}_\frob^2 = \stocho(1)$.

    \item \label{enum:trace_init:w}
    $\frac{1}{h^2}\tr(\Wbf^\top\Wbf(0))^2 = \stochO(1)$.
  \end{enumerate}
\end{lemma}

\begin{proof}[Proof of \cref{lemma:frobenius_norm_init}]
  Let us prove \ref{enum:frobenius_norm_init:phi}.
  Again, each element of the random matrix $\sqrt{d}\Phibf(0)$ is $2$-sub-Weibull (see the proof of \cref{lemma:spectral_norm_init}).
  Thus, \cref{lemma:frobenius_norm_bound} implies $\frac{1}{h^2}\norm{\Phibf^\top\Phibf(0)}_\frob^2 = \frac{1}{d^2h^2}\cdot d^2\norm{\Phibf^\top\Phibf(0)}_\frob^2 = \stochO(1/h) = \stocho(1)$.

  The identity \ref{enum:frobenius_norm_init:phiw} follows similarly.
  Again, each element of the random matrix $\sqrt{dh}\Wbf\Phibf(0)$ is $1$-sub-Weibull (see the proof of \cref{lemma:spectral_norm_init}) so that $\sqrt{dh}\Wbf\Phibf(0)$ satisfies the assumption of \cref{lemma:frobenius_norm_bound},
  from which we deduce that
  \begin{equation*}
    \begin{aligned}
      \frac{1}{h^2}\norm{\Phibf^\top\Wbf^\top\Wbf\Phibf(0)}_\frob^2
      &= \frac{1}{h^2} \cdot \frac{1}{d^2h^2} \cdot d^2h^2\norm{\Phibf^\top\Wbf^\top\Wbf\Phibf(0)}_\frob^2 \\
      &= \frac{1}{h^2} \stochO(1) \\
      &= \stocho(1).
    \end{aligned}
  \end{equation*}

  To prove \ref{enum:trace_init:w}, we see that
  $h\tr(\Wbf^\top\Wbf(0)) = h\norm{\Wbf(0)}_\frob^2 = \sum_{i,j=1}^h(\sqrt{h}W(0)_{ij})^2$
  is the sum of sub-exponential (namely, $1$-sub-Weibull) random variables $(\sqrt{h}W(0)_{ij})^2$ with $\norm{\sqrt{h}W(0)_{ij}}_{\psi_1} = \Ocal(1)$ for $i, j \in [h]$.
  Hence, $h\tr(\Wbf^\top\Wbf(0))$ is $1$-sub-Weibull with $\norm{h\tr(\Wbf^\top\Wbf(0))}_{\psi_1} = \Ocal(h^2)$ from \cref{proposition:sub_weibull_closure}.
  By the closure property again, $h^2\tr(\Wbf^\top\Wbf(0))^2$ is $\frac{1}{2}$-sub-Weibull with the corresponding norm being $\Ocal(h^4)$.
  By using sub-Weibull tails in \cref{proposition:sub_weibull},
  we deduce that $\abs{h^2\tr(\Wbf^\top\Wbf(0))^2} = \stochO(h^2)$.
  Lastly, we obtain $\frac{1}{h^2}\tr(\Wbf^\top\Wbf(0))^2 = \stochO(1)$.
\end{proof}

\begin{lemma}
  \label{lemma:l2_norm_init}
  Under \cref{assumption:standard_mvn,assumption:initialization}, we have the following consequences:
  \begin{enumerate}
    \item \label{enum:l2_norm_init:phix}
    $\frac{1}{h^2}\norm{\Phibf^\top\Phibf(0)\xbf_0}_2^2 = \stocho(1)$.

    \item \label{enum:l2_norm_init:wphix}
    $\frac{1}{h^4}\norm{\Phibf^\top\Wbf^\top\Wbf\Phibf(0)\xbf_0}_2^2 = \stocho(1)$.
  \end{enumerate}
\end{lemma}

\begin{proof}[Proof of \cref{lemma:l2_norm_init}]
  \Cref{assumption:standard_mvn} implies that $\xbf_0 \sim \Ncal(\zerobf, \Ibf_d)$,
  from which we can verify that $\norm{\xbf_0}_2^2 = \sum_{i=1}^dx_{0,i}^2$ is the sum of $d$ sub-exponential (i.e., $1$-sub-Weibull) random variables
  and $\norm{\norm{\xbf_0}_2^2}_{\psi_1} = \Ocal(d)$ (\cref{proposition:sub_weibull_closure}).
  By sub-Weibull tails (\cref{proposition:sub_weibull}), $\norm{\xbf_0}_2^2 = \stochO(d)$ entails.

  To prove \ref{enum:l2_norm_init:phix}, we confirm that each element of $h\Phibf^\top\Phibf(0)$ is sub-exponential with the $\psi_1$-norm being $\Ocal(1)$.
  To see this, we let $\Phibf_i$ denote the $i$-th column vector of $\sqrt{h}\Phibf(0)$.
  \Cref{assumption:initialization} indicates that $\Phibf_i$ is an $h$-dimensional standard normal random vector,
  and $\E\inpr{\Phibf_i}{\Phibf_j} = h \cdot \iverson{i=j}$.
  Thus, Bernstein's inequality \citep[Corollary 2.8.3]{Vershynin2018} yields
  $\abs{\inpr{\Phibf_i}{\Phibf_j} - h \cdot \iverson{i=j}} = \stochO(1)$ (for sufficiently large $h$),
  which indicates that $h\Phibf^\top\Phibf(0) - h\Ibf_d$ satisfies the assumption of \cref{lemma:spectral_norm_bound} with $\beta=1$ and $K(d,h)=1$.
  Hence, by \cref{lemma:spectral_norm_bound},
  \begin{equation*}
    \norm{h\Phibf^\top\Phibf(0)}
    \le \norm{h\Phibf^\top\Phibf(0) - h\Ibf_d} + h\norm{\Ibf_d}
    = \stochO(d) + h.
  \end{equation*}
  Combining this with $\norm{\xbf_0}_2^2 = \stochO(d)$, we obtain the following result:
  \begin{equation*}
    \frac{1}{h^2}\norm{\Phibf^\top\Phibf(0)\xbf_0}_2^2
    \le \frac{1}{h^4} \cdot \norm{h\Phibf^\top\Phibf(0)}^2 \cdot \norm{\xbf_0}_2^2
    = \frac{1}{h^4} \cdot \left\{\stochO(d) + h\right\}^2 \cdot \stochO(d)
    = \stochO(h^{-1}),
  \end{equation*}
  which completes the proof.
  
  To prove \ref{enum:l2_norm_init:wphix}, we confirm that each element of $h^2\Phibf^\top\Wbf^\top\Wbf\Phibf(0)$ is $\frac{1}{2}$-sub-Weibull with the $\psi_{\frac{1}{2}}$-norm being $\Ocal(\sqrt{h})$.
  To see this, we let $\Psibf_i$ denote the $i$-th column vector of $h\Wbf(0)\Phibf(0)$ (for $i \in [d]$).
  The $k$-th element of $\Psibf_i$ (for $k \in [h]$) is $\Psi_i^{(k)} \defeq h\sum_{l=1}^hW(0)_{kl}\Phi(0)_{li}$,
  which is sub-exponential and mean zero from \cref{assumption:initialization} and $\abs{\Psi_i^{(k)}} = \abs{h\sum_{l=1}^hW(0)_{kl}\Phi(0)_{li}} = \stochO(1)$ (for sufficiently large $h$) from Bernstein's inequality.
  Here, each $(i,j)$-th element of $h^2\Phibf^\top\Wbf^\top\Wbf\Phibf(0)$ is $\inpr{\Psibf_i}{\Psibf_j} = \sum_{k=1}^h\Psi_i^{(k)}\Psi_j^{(k)}$,
  which is the sum of $h$ products $\Psi_i^{(k)}\Psi_j^{(k)}$.
  Each $\Psi_i^{(k)}\Psi_j^{(k)}$ is $\frac{1}{2}$-sub-Weibull because of the closure property (\cref{proposition:sub_weibull_closure}),
  and hence the sum $\inpr{\Psibf_i}{\Psibf_j}$ is $\frac{1}{2}$-sub-Weibull with $\norm{\inpr{\Psibf_i}{\Psibf_j}}_{\psi_\frac{1}{2}} = \Ocal(h)$.
  Thus, we see the sub-Weibull property of $h^2\Phibf^\top\Wbf^\top\Wbf\Phibf(0)$.
  Hence, we can apply \cref{lemma:spectral_norm_bound} to claim $\norm{h^2\Phibf^\top\Wbf^\top\Wbf\Phibf(0)} = \stochO(d^2h)$.
  Combining this with $\norm{\xbf_0}_2^2 = \stochO(d)$, we obtain the desired result:
  \begin{equation*}
    \begin{aligned}
      \frac{1}{h^4}\norm{\Phibf^\top\Wbf^\top\Wbf\Phibf(0)\xbf_0}_2^2
      &\le \frac{1}{h^8} \cdot \norm{h^2\Phibf^\top\Wbf^\top\Wbf\Phibf(0)}^2 \cdot \norm{\xbf_0}_2^2 \\
      &= \frac{1}{h^8} \cdot \stochO(d^4h^2) \cdot \stochO(d) \\
      &= \stochO(h^{-1}).
    \end{aligned}
  \end{equation*}
\end{proof}

\begin{lemma}
  \label{lemma:integral_inequality_frobenius_norm_phi}
  For any $t$, $\norm{\Phibf^\top\Phibf(t)}_\frob \le (\norm{\Phibf^\top\Phibf(0)}_\frob + 4t) \exp(2\rho t)$.
\end{lemma}

\begin{proof}[Proof of \cref{lemma:integral_inequality_frobenius_norm_phi}]
  First, we use the fundamental theorem of calculus and the triangular inequality to decompose as follows:
  \begin{equation}
    \begin{aligned}
      \norm{\Phibf^\top\Phibf(t)}_\frob
      &= \norm{\Phibf^\top\Phibf(0) + \int_0^t \left\{ \dot\Phibf^\top\Phibf(\tau) + \Phibf^\top\dot\Phibf(\tau) \right\}\rd{\tau}}_\frob \\
      &\le \norm{\Phibf^\top\Phibf(0)}_\frob + \int_0^t \norm{\dot\Phibf^\top\Phibf(\tau)}_\frob \rd{\tau} + \int_0^t \norm{\Phibf^\top\dot\Phibf(\tau)}_\frob \rd{\tau} \\
      &= \norm{\Phibf^\top\Phibf(0)}_\frob + 2 \int_0^t \norm{\dot\Phibf^\top\Phibf(\tau)}_\frob \rd{\tau}.
    \end{aligned}
    \label{equation:proof:time_evolution_phi_norm}
  \end{equation}
  The term $\dot\Phibf^\top\Phibf$ can be evaluated by using the dynamics derived in \cref{lemma:dynamics_matrix_expected} as follows:
  \begin{equation}
    \begin{aligned}
      \dot\Phibf^\top\Phibf
      &= \left\{ \Wbf^\top (\Wbf^\top\dot\Wbf + \rho\Wbf\Wbf^\top)\pinv{(\Wbf^\top)}\pinv{(\Phibf^\top)} - \rho\Phibf\Phibf^\top\Wbf^\top\pinv{(\Wbf^\top)}\pinv{(\Phibf^\top)} \right\}^\top \Phibf \\
      &= \left\{ \Wbf^\top\Wbf^\top\dot\Wbf(\pinv{\Wbf})^\top(\pinv{\Phibf})^\top + \rho\Wbf^\top\Wbf(\pinv{\Phibf})^\top - \rho\Phibf \right\}^\top \Phibf \\
      &= \pinv{\Phibf}\pinv{\Wbf}\dot\Wbf^\top\Wbf^2\Phibf + \rho\pinv{\Phibf}\Wbf^\top\Wbf\Phibf - \rho\Phibf^\top\Phibf \\
      &= \pinv{\Phibf}\pinv{\Wbf}\{\E[\zbf'\omegabf^\top - (\omegabf^\top\zbf')\omegabf\omegabf^\top] - \rho\Wbf\Wbf^\top\}\Wbf\Phibf + \rho\pinv{\Phibf}\Wbf^\top\Wbf\Phibf - \rho\Phibf^\top\Phibf \\
      &= \pinv{\Phibf}\pinv{\Wbf} \E[\zbf'\omegabf^\top - (\omegabf^\top\zbf')\omegabf\omegabf^\top] \Wbf\Phibf - \rho\Phibf^\top\Phibf,
    \end{aligned}
    \label{equation:proof:dot_phit_phi}
  \end{equation}
  whose Frobenius norm shall be bounded from above subsequently:
  \begin{equation*}
      \norm{\dot\Phibf^\top\Phibf}_\frob
      \le \E\norm{\pinv{\Phibf}\pinv{\Wbf}(\zbf'\omegabf^\top)\Wbf\Phibf}_\frob + \E\norm{\pinv{\Phibf}\pinv{\Wbf}(\omegabf\omegabf^\top)\Wbf\Phibf}_\frob + \rho\norm{\Phibf^\top\Phibf}_\frob.
  \end{equation*}
  Note that we use $\abs{\omegabf^\top\zbf'} \le 1$ because $\omegabf, \zbf' \in \Sbb^{h-1}$ in this bound.
  The norm $\norm{\pinv{\Phibf}\pinv{\Wbf}(\zbf'\omegabf^\top)\Wbf\Phibf}_\frob$ is bounded as follows:
  \begin{equation}
    \begin{aligned}
      \norm{\pinv{\Phibf}\pinv{\Wbf}(\zbf'\omegabf^\top)\Wbf\Phibf}_\frob^2
      &= \inpr{\pinv{\Phibf}\pinv{\Wbf}(\zbf'\omegabf^\top)\Wbf\Phibf}{\pinv{\Phibf}\pinv{\Wbf}(\zbf'\omegabf^\top)\Wbf\Phibf}_\frob \\
      &= \tr(\Phibf^\top\Wbf^\top\omegabf(\zbf')^\top(\pinv{\Wbf})^\top(\pinv{\Phibf})^\top\pinv{\Phibf}\pinv{\Wbf}\zbf'\omegabf^\top\Wbf\Phibf) \\
      &\overset{\text{(*)}}= \tr(\omegabf(\zbf')^\top(\pinv{\Wbf})^\top(\pinv{\Phibf})^\top\pinv{\Phibf}\pinv{\Wbf}\zbf'\omegabf^\top\Wbf\Phibf\Phibf^\top\Wbf^\top) \\
      &\le \abs{\tr(\omegabf(\zbf')^\top)} \cdot \abs{\tr((\pinv{\Wbf})^\top(\pinv{\Phibf})^\top\pinv{\Phibf}\pinv{\Wbf}\zbf'\omegabf^\top\Wbf\Phibf\Phibf^\top\Wbf^\top)} \\
      &\overset{\text{(*)}}= \abs{\tr(\omegabf(\zbf')^\top)} \cdot \abs{\tr(\zbf'\omegabf^\top)} \\
      &= \norm{\omegabf}_2 \norm{\zbf'}_2 \norm{\zbf'}_2 \norm{\omegabf}_2 \\
      &\le 1,
    \end{aligned}
    \label{equation:proof:bound_gradient_norm}
  \end{equation}
  where the cyclic property of the trace $\tr(\Abf\Bbf\Cbf) = \tr(\Bbf\Cbf\Abf)$ is used at the two identities (*).
  Because \cref{equation:proof:bound_gradient_norm} relies solely on $\zbf', \omegabf \in \Sbb^{h-1}$,
  the same reasoning induces the upper bound $\norm{\pinv{\Phibf}\pinv{\Wbf}(\omegabf\omegabf^\top)\Wbf\Phibf}_\frob \le 1$.
  By plugging everything back to \cref{equation:proof:time_evolution_phi_norm},
  we obtain the following integral inequality for the norm $\norm{\Phibf^\top\Phibf(t)}_\frob$:
  \begin{equation}
    \norm{\Phibf^\top\Phibf(t)}_\frob \le \norm{\Phibf^\top\Phibf(0)}_\frob + 4t + 2\rho\int_0^t \norm{\Phibf^\top\Phibf(\tau)}_\frob \rd{\tau}.
    \label{equation:proof:integral_inequality_phi_norm}
  \end{equation}
  The form of \cref{equation:proof:integral_inequality_phi_norm} satisfies the assumption of the Gr\"{o}nwall--Bellman inequality (\cref{theorem:gronwall_bellman})
  with which the norm upper bound $\norm{\Phibf^\top\Phibf(t)}_\frob \le (\norm{\Phibf^\top\Phibf(0)}_\frob + 4t) \exp(2\rho t)$ is derived.
\end{proof}

\begin{lemma}
  \label{lemma:integral_inequality_spectral_norm_phi}
  For any $t$, $\norm{\Phibf(t)} \le \sqrt{(\norm{\Phibf^\top\Phibf(0)} + 4t) \exp(2\rho t)}$.
\end{lemma}

\begin{proof}[Proof of \cref{lemma:integral_inequality_spectral_norm_phi}]
  We evaluate $\norm{\Phibf^\top\Phibf(t)} = \norm{\Phibf(t)}^2$.
  By the fundamental theorem of calculus, we obtain the following decomposition:
  \begin{equation*}
    \norm{\Phibf^\top\Phibf(t)}
    \le \norm{\Phibf^\top\Phibf(0)} + 2 \int_0^t \norm{\dot\Phibf^\top\Phibf(\tau)}\rd{\tau}.
  \end{equation*}
  By following the same derivation as the proof of \cref{lemma:integral_inequality_frobenius_norm_phi}, it is not difficult to see $\norm{\dot\Phibf^\top\Phibf} \le 2 + \rho\norm{\Phibf^\top\Phibf}$.
  Then, $\norm{\Phibf^\top\Phibf(t)} \le \norm{\Phibf^\top\Phibf(0)} + 4t + 2\rho\int_0^t\norm{\Phibf^\top\Phibf(\tau)}\rd{\tau}$.
  This integral inequality can be solved via \cref{theorem:gronwall_bellman},
  and we have $\norm{\Phibf^\top\Phibf(t)} \le (\norm{\Phibf^\top\Phibf(0)} + 4t)\exp(2\rho t)$.
\end{proof}

\begin{lemma}
  \label{lemma:integral_inequality_tr_w}
  For $\Wbf \in \sym_h$,
  for any $t$,
  \[
    \tr(\Wbf^\top\Wbf(t)) \le (\tr(\Wbf^\top\Wbf(0)) + 4t)\exp(2\rho t).
  \]
\end{lemma}

\begin{proof}[Proof of \cref{lemma:integral_inequality_tr_w}]
  By the fundamental theorem of calculus, we obtain the following decomposition:
  \begin{equation*}
    \tr(\Wbf^\top\Wbf(t))
    \le \tr(\Wbf^\top\Wbf(0)) + 2\int_0^t\tr(\Wbf^\top\dot\Wbf(\tau))\rd{\tau}.
  \end{equation*}
  By using the dynamics in \cref{lemma:dynamics_matrix_expected}, we further obtain the bound of $\tr(\Wbf^\top\dot\Wbf)$:
  \begin{equation*}
    \begin{aligned}
      \tr(\Wbf^\top\dot\Wbf)
      &= \tr\left(\E[\zbf'\omegabf^\top - (\omegabf^\top\zbf')\omegabf\omegabf^\top] - \rho\Wbf\Wbf^\top\right) \\
      &\le \E\tr(\zbf'\omegabf^\top) + \E\tr(\omegabf\omegabf^\top) + \rho\tr(\Wbf\Wbf^\top) \\
      &\le 2 + \rho\tr(\Wbf^\top\Wbf),
    \end{aligned}
  \end{equation*}
  where the trace evaluation of rank-$1$ matrices and the symmetry of $\Wbf$ are used.
  Hence, we obtain the following integral inequality:
  \begin{equation*}
    \tr(\Wbf^\top\Wbf(t)) \le \tr(\Wbf^\top\Wbf(0)) + 4t + 2\rho\int_0^t\tr(\Wbf^\top\Wbf(\tau))\rd{\tau},
  \end{equation*}
  which is the same form as the integral inequality in \cref{equation:proof:integral_inequality_phi_norm},
  and can be solved in the same way.
\end{proof} 

\begin{lemma}
  \label{lemma:integral_inequality_l2_norm_phix}
  For any $t$, $\norm{\Phibf^\top\Phibf(t)\xbf_0}_2^2 \le (\norm{\Phibf^\top\Phibf(0)\xbf_0}_2^2 + 4\norm{\xbf_0}_2^2t)\exp(2\rho t)$.
\end{lemma}

\begin{proof}[Proof of \cref{lemma:integral_inequality_l2_norm_phix}]
  First, we obtain
  \begin{equation*}
    \norm{\Phibf^\top\Phibf(t)\xbf_0}_2^2
    \le \norm{\Phibf^\top\Phibf(0)\xbf_0}_2^2 + \int_0^t\norm{\dot\Phibf^\top\Phibf(\tau)\xbf_0}_2^2\rd{\tau} + \int_0^t\norm{\Phibf^\top\dot\Phibf(\tau)\xbf_0}_2^2\rd{\tau},
  \end{equation*}
  which is obtained in the same manner as \cref{equation:proof:time_evolution_phi_norm} (in the proof of \cref{lemma:integral_inequality_frobenius_norm_phi}).
  We substitute the dynamics (\cref{lemma:dynamics_matrix_expected}), or \cref{equation:proof:dot_phit_phi} in the proof of \cref{lemma:integral_inequality_frobenius_norm_phi}, and simplify $\norm{\dot\Phibf^\top\Phibf(\tau)\xbf_0}_2^2$ as follows:
  \begin{equation*}
    \begin{aligned}
      \norm{\dot\Phibf^\top\Phibf\xbf_0}_2^2
      &= \norm{\pinv{\Phibf}\pinv{\Wbf}\E[\zbf'\omegabf^\top - (\omegabf^\top\zbf')\omegabf\omegabf^\top]\Wbf\Phibf\xbf_0 - \rho\Phibf^\top\Phibf\xbf_0}_2^2 \\
      &\le \E\norm{\pinv{\Phibf}\pinv{\Wbf}(\zbf'\omegabf^\top)\Wbf\Phibf\xbf_0}_2^2 + \E\norm{\pinv{\Phibf}\pinv{\Wbf}(\omegabf\omegabf^\top)\Wbf\Phibf\xbf_0}_2^2 \\ &\phantom{\le} + \rho \norm{\Phibf^\top\Phibf\xbf_0}_2^2,
    \end{aligned}
  \end{equation*}
  where $\abs{\omegabf^\top\zbf'} \le 1$ is used.
  The first term is bounded as follows:
  \begin{equation*}
    \begin{aligned}
      &\!\!\norm{\pinv{\Phibf}\pinv{\Wbf}(\zbf'\omegabf^\top)\Wbf\Phibf\xbf_0}_2^2 \\
      &= \tr\left(\pinv{\Phibf}\pinv{\Wbf}(\zbf'\omegabf^\top)\Wbf\Phibf\xbf_0\xbf_0^\top\Phibf^\top\Wbf^\top(\omegabf(\zbf')^\top)(\pinv{\Wbf})^\top(\pinv{\Phibf})^\top\right) \\
      &\overset{\text{(*)}}= \tr\left((\zbf'\omegabf^\top)\Wbf\Phibf\xbf_0\xbf_0^\top\Phibf^\top\Wbf^\top(\omegabf(\zbf')^\top)(\pinv{\Wbf})^\top(\pinv{\Phibf})^\top\pinv{\Phibf}\pinv{\Wbf}\right) \\
      &\overset{(\natural)}\le \abs{\tr\left(\Wbf\Phibf\xbf_0\xbf_0^\top\Phibf^\top\Wbf^\top(\omegabf(\zbf')^\top)(\pinv{\Wbf})^\top(\pinv{\Phibf})^\top\pinv{\Phibf}\pinv{\Wbf}\right)} \\
      &\overset{\text{(*)}}= \abs{\tr\left((\omegabf(\zbf')^\top)(\pinv{\Wbf})^\top(\pinv{\Phibf})^\top\pinv{\Phibf}\pinv{\Wbf}\Wbf\Phibf\xbf_0\xbf_0^\top\Phibf^\top\Wbf^\top\right)} \\
      &\overset{(\natural)}\le \abs{\tr\left((\pinv{\Wbf})^\top(\pinv{\Phibf})^\top\pinv{\Phibf}\pinv{\Wbf}\Wbf\Phibf\xbf_0\xbf_0^\top\Phibf^\top\Wbf^\top\right)} \\
      &\overset{\text{(*)}}= \abs{\tr\left(\pinv{\Phibf}\pinv{\Wbf}\Wbf\Phibf\xbf_0\xbf_0^\top\right)} \\
      &\le \abs{\tr\left(\pinv{\Phibf}\pinv{\Wbf}\Wbf\Phibf\right) \cdot \tr\left(\xbf_0\xbf_0^\top\right)} \\
      &\overset{\text{(*)}}= \abs{\tr\left(\xbf_0\xbf_0^\top\right)} \\
      &= \norm{\xbf_0}_2^2,
    \end{aligned}
  \end{equation*}
  where we use the trace cyclic property at (*),
  and the Cauchy-Schwartz inequality and the trace property $\abs{\tr(\zbf\omegabf^\top)} = \abs{\omegabf^\top\zbf} \le 1$ for $\zbf, \omegabf \in \Sbb^{h-1}$ at ($\natural$).
  Similarly, $\norm{\pinv{\Phibf}\pinv{\Wbf}(\omegabf\omegabf^\top)\Wbf\Phibf\xbf_0}_2^2 \le \norm{\xbf_0}_2^2$.
  Thus, we have $\norm{\dot\Phibf^\top\Phibf\xbf_0}_2^2 \le 2\norm{\xbf_0}_2^2 + \rho\norm{\Phibf^\top\Phibf\xbf_0}_2^2$.
  By doing the same algebra again, we have $\norm{\Phibf^\top\dot\Phibf\xbf_0}_2^2 \le 2\norm{\xbf_0}_2^2 + \rho\norm{\Phibf^\top\Phibf\xbf_0}_2^2$ as well.
  By combining them,
  \begin{equation*}
    \norm{\Phibf^\top\Phibf(t)\xbf_0}_2^2 \le \norm{\Phibf^\top\Phibf(0)\xbf_0}_2^2 + 4\norm{\xbf_0}_2^2t + 2\rho\int_0^t\norm{\Phibf^\top\Phibf(\tau)\xbf_0}_2^2 \rd{\tau}
  \end{equation*}
  holds, to which the Gr\"{o}nwall--Bellman inequality (\cref{theorem:gronwall_bellman}) can be used,
  and we deduce $\norm{\Phibf^\top\Phibf(t)\xbf_0}_2^2 \le (\norm{\Phibf^\top\Phibf(0)\xbf_0}_2^2 + 4\norm{\xbf_0}_2^2t)\exp(2\rho t)$. 
\end{proof}

\begin{lemma}
  \label{lemma:integral_inequality_frobenius_norm_phiw}
  For $\Wbf \in \sym_h$,
  for any $t$, the following bound holds:
  \begin{equation*}
    \begin{aligned}
      &\norm{\Phibf^\top\Wbf^\top\Wbf\Phibf(t)}_\frob \\
      &\le \left\{ \norm{\Phibf^\top\Wbf^\top\Wbf\Phibf(0)}_\frob + \frac{16\rho te^{2\rho t} + (2\rho I_0 - 8)(e^{2\rho t} - 1)}{\rho^2} \right\}e^{4\rho t},
    \end{aligned}
  \end{equation*}
  where $I_0 \defeq \tr(\Wbf^\top\Wbf(0)) + \norm{\Phibf^\top\Phibf(0)}_\frob$.
\end{lemma}

\begin{proof}[Proof of \cref{lemma:integral_inequality_frobenius_norm_phiw}]
  By using the fundamental theorem of calculus and the triangular inequality, the Frobenius norm $\norm{\Phibf^\top\Wbf^\top\Wbf\Phibf(t)}_\frob$ is bounded:
  \begin{equation}
    \begin{aligned}
      &\norm{\Phibf^\top\Wbf^\top\Wbf\Phibf(t)}_\frob \\
      &\le \norm{\Phibf^\top\Wbf^\top\Wbf\Phibf(0)}_\frob + 2\int_0^t \norm{\diff{(\Wbf\Phibf)(\tau)}{\tau}^\top(\Wbf\Phibf)(\tau)}_\frob\rd{\tau} \\
      &\le \norm{\Phibf^\top\Wbf^\top\Wbf\Phibf(0)}_\frob + 2\underbrace{\int_0^t\norm{\dot\Phibf^\top\Wbf^\top\Wbf\Phibf(\tau)}_\frob\rd{\tau}}_\text{(A)} \\ &\phantom{\le} + 2\underbrace{\int_0^t\norm{\Phibf^\top\dot\Wbf^\top\Wbf\Phibf(\tau)}_\frob\rd{\tau}}_\text{(B)}.
    \end{aligned}
    \label{equation:proof:time_evolution_wphi_norm}
  \end{equation}
  To bound (A) in \cref{equation:proof:time_evolution_wphi_norm}, we proceed by plugging the dynamics (\cref{lemma:dynamics_matrix_expected}) in as follows:
  \begin{equation}
    \begin{aligned}
      &\!\!\norm{\dot\Phibf^\top\Wbf^\top\Wbf\Phibf}_\frob \\
      &= \norm{(\pinv{\Phibf}\pinv{\Wbf}\E[\omegabf(\zbf')^\top - (\omegabf^\top\zbf')\omegabf\omegabf^\top]\Wbf - \rho\Phibf^\top)\Wbf^\top\Wbf\Phibf}_\frob \\
      &\le \underbrace{\E\norm{\pinv{\Phibf}\pinv{\Wbf}(\omegabf(\zbf')^\top)\Wbf\Wbf^\top\Wbf\Phibf}_\frob}_{(\clubsuit)} + \underbrace{\E\norm{\pinv{\Phibf}\pinv{\Wbf}(\omegabf\omegabf^\top)\Wbf\Wbf^\top\Wbf\Phibf}_\frob}_{(\diamondsuit)} \\  
        &\phantom{\le} + \rho\norm{\Phibf^\top\Wbf^\top\Wbf\Phibf}_\frob.
    \end{aligned}
    \label{equation:proof:norm_dotphi_w_w_phi}
  \end{equation}
  We bound the squared ($\clubsuit$) in \cref{equation:proof:norm_dotphi_w_w_phi} as follows:
  \begin{equation*}
    \begin{aligned}
      & \norm{\pinv{\Phibf}\pinv{\Wbf}(\omegabf(\zbf')^\top)\Wbf\Wbf^\top\Wbf\Phibf}_\frob^2 \\
      &= \tr\left(\Phibf^\top\Wbf^\top\Wbf\Wbf^\top(\zbf'\omegabf^\top)(\pinv{\Wbf})^\top(\pinv{\Phibf})^\top \cdot \pinv{\Phibf}\pinv{\Wbf}(\omegabf(\zbf')^\top)\Wbf\Wbf^\top\Wbf\Phibf\right) \\
      &\overset{\text{(*$\natural$)}}\le \abs{\tr\left((\pinv{\Wbf})^\top(\pinv{\Phibf})^\top\pinv{\Phibf}\pinv{\Wbf}(\omegabf(\zbf')^\top)\Wbf\Wbf^\top\Wbf\Phibf \cdot \Phibf^\top\Wbf^\top\Wbf\Wbf^\top\right)} \\
      &\overset{\text{(*)}}= \abs{\tr\left((\pinv{\Phibf})^\top\pinv{\Phibf}\pinv{\Wbf}(\omegabf(\zbf')^\top)\Wbf\Wbf^\top\Wbf\Phibf\Phibf^\top\Wbf^\top\Wbf\right)} \\
      &\overset{\text{(*$\natural$)}}\le \abs{\tr\left(\Wbf\Wbf^\top\Wbf\Phibf\Phibf^\top\Wbf^\top\Wbf \cdot (\pinv{\Phibf})^\top\pinv{\Phibf}\pinv{\Wbf}\right)} \\
      &\overset{\text{(*)}}= \abs{\tr\left(\pinv{\Phibf}\Wbf^\top\Wbf\Phibf \cdot \Phibf^\top\Wbf^\top\Wbf(\pinv{\Phibf})^\top\right)} \\
      &\le \abs{\tr(\pinv{\Phibf}\Wbf^\top\Wbf\Phibf) \cdot \tr(\Phibf^\top\Wbf^\top\Wbf(\pinv{\Phibf})^\top)} \\
      &\overset{\text{(*)}}= \tr(\Wbf^\top\Wbf)^2,
    \end{aligned}
  \end{equation*}
  where we use the trace cyclic property at (*),
  and use the trace cyclic property, the Cauchy-Schwartz inequality, and the trace evaluation of rank-$1$ matrices at (*$\natural$), as we do in the proof of \cref{lemma:integral_inequality_l2_norm_phix}.
  By using the same techniques, the squared ($\diamondsuit$) in \cref{equation:proof:norm_dotphi_w_w_phi} can be bounded by $\tr(\Wbf^\top\Wbf)$ as well.
  Hence, we obtain the bound of \cref{equation:proof:norm_dotphi_w_w_phi} as
  $\norm{\dot\Phibf^\top\Wbf^\top\Wbf\Phibf}_\frob \le 2\tr(\Wbf^\top\Wbf) + \rho\norm{\Phibf^\top\Wbf^\top\Wbf\Phibf}_\frob$.
  To bound (B) in \cref{equation:proof:time_evolution_wphi_norm}, the dynamics (\cref{lemma:dynamics_matrix_expected}) is plugged in again:
  \begin{equation}
    \begin{aligned}
      \norm{\Phibf^\top\dot\Wbf^\top\Wbf\Phibf}_\frob
      &= \norm{\Phibf^\top\E[\omegabf(\zbf')^\top - (\omegabf^\top\zbf')\omegabf\omegabf^\top]\Phibf - \rho\Phibf^\top\Wbf\Wbf^\top\Phibf^\top} \\
      &\le \underbrace{\E\norm{\Phibf^\top(\omegabf(\zbf')^\top)\Phibf}_\frob}_{(\heartsuit)} + \underbrace{\E\norm{\Phibf^\top(\omegabf\omegabf^\top)\Phibf}_\frob}_{(\spadesuit)} + \rho\norm{\Phibf^\top\Wbf\Wbf^\top\Phibf}_\frob,
    \end{aligned}
    \label{equation:proof:norm_phi_dotw_w_phi}
  \end{equation}
  where the squared ($\heartsuit$) is bounded as follows:
  \begin{equation*}
    \begin{aligned}
      \norm{\Phibf^\top(\omegabf(\zbf')^\top)\Phibf}_\frob^2
      &= \tr\left(\Phibf^\top(\zbf'\omegabf^\top)\Phibf \cdot \Phibf^\top(\omegabf(\zbf')^\top)\Phibf\right) \\
      &\overset{(*\natural)}\le \abs{\tr\left(\Phibf\Phibf^\top(\omegabf(\zbf')^\top)\Phibf\Phibf^\top\right)} \\
      &\overset{(*\natural)}\le \abs{\tr\left(\Phibf\Phibf^\top\Phibf\Phibf^\top\right)} \\
      &= \norm{\Phibf\Phibf^\top}_\frob^2 \\
      &= \norm{\Phibf^\top\Phibf}_\frob^2.
    \end{aligned}
  \end{equation*}
  The squared ($\spadesuit$) is bounded by $\norm{\Phibf^\top\Phibf}_\frob$ as well.
  Hence, we obtain the bound of \cref{equation:proof:norm_phi_dotw_w_phi} as $\norm{\Phibf^\top\dot\Wbf^\top\Wbf\Phibf}_\frob \le 2\norm{\Phibf^\top\Phibf}_\frob + \rho\norm{\Phibf^\top\Wbf\Wbf^\top\Phibf}_\frob$.
  Eventually, we obtain the following bound from \cref{equation:proof:time_evolution_wphi_norm} (which requires the symmetry of $\Wbf$):
  \begin{equation*}
    \begin{aligned}
      \norm{\Phibf^\top\Wbf^\top\Wbf\Phibf(t)}_\frob
      &\le \norm{\Phibf^\top\Wbf^\top\Wbf\Phibf(0)}_\frob + 4\int_0^t\tr(\Wbf^\top\Wbf(\tau))\rd{\tau} \\
        &\phantom{\le} + 4\int_0^t\norm{\Phibf^\top\Phibf(\tau)}_\frob\rd{\tau} + 4\int_0^t\rho\norm{\Phibf^\top\Wbf^\top\Wbf\Phibf(\tau)}_\frob\rd{\tau} \\
      &\le \norm{\Phibf^\top\Wbf^\top\Wbf\Phibf(0)}_\frob + 4\int_0^t\rho\norm{\Phibf^\top\Wbf^\top\Wbf\Phibf(\tau)}_\frob\rd{\tau} \\
        &\phantom{\le} + 4\int_0^t \left\{ I_0\exp(2\rho\tau) + 8\tau\exp(2\rho\tau) \right\}\rd{\tau} \\
      &\le \norm{\Phibf^\top\Wbf^\top\Wbf\Phibf(0)}_\frob + 4\int_0^t\rho\norm{\Phibf^\top\Wbf^\top\Wbf\Phibf(\tau)}_\frob\rd{\tau} \\
        &\phantom{\le} + \frac{16\rho te^{2\rho t} + (2\rho I_0 - 8)(e^{2\rho t} - 1)}{\rho^2},
    \end{aligned}
  \end{equation*}
  where \cref{lemma:integral_inequality_frobenius_norm_phi,lemma:integral_inequality_tr_w} are used at the second inequality
  and integration by parts is used in the third inequality.
  This integral inequality can be solved by the Gr\"{o}nwall--Bellman inequality (\cref{theorem:gronwall_bellman}), and we can obtain the conclusion.
\end{proof}

\begin{lemma}
  \label{lemma:integral_inequality_spectral_norm_phiw}
  For $\Wbf \in \sym_h$,
  for any $t$, the following bound holds:
  \begin{equation*}
    \norm{\Wbf\Phibf(t)}
    \le \sqrt{\left\{ \norm{\Phibf^\top\Wbf^\top\Wbf\Phibf(0)} + \frac{16\rho te^{2\rho t} + (2\rho I_0 - 8)(e^{2\rho t} - 1)}{\rho^2} \right\}e^{4\rho t}},
  \end{equation*}
  where $I_0$ is defined in \cref{lemma:integral_inequality_frobenius_norm_phiw}.
\end{lemma}

\begin{proof}[Proof of \cref{lemma:integral_inequality_spectral_norm_phiw}]
  We evaluate $\norm{\Phibf^\top\Wbf^\top\Wbf\Phibf(t)} = \norm{\Wbf\Phibf(t)}^2$.
  By the fundamental theorem of calculus, we obtain the following decomposition:
  \begin{equation*}
    \norm{\Phibf^\top\Wbf^\top\Wbf\Phibf(t)} \le \norm{\Phibf^\top\Wbf^\top\Wbf\Phibf(0)} + 2\int_0^t \norm{\left(\diff{\Wbf\Phibf}{\tau}\right)^\top \Wbf\Phibf(\tau)}\rd{\tau}.
  \end{equation*}
  By following the same derivation as the proof of \cref{lemma:integral_inequality_frobenius_norm_phiw}, it is not difficult to see the following upper bound:
  \begin{equation*}
    \norm{\left(\diff{\Wbf\Phibf}{\tau}\right)^\top \Wbf\Phibf}
    \le 2\tr(\Wbf^\top\Wbf) + 2\norm{\Phibf^\top\Phibf}_\frob + 2\rho\norm{\Phibf^\top\Wbf^\top\Wbf\Phibf}.
  \end{equation*}
  By plugging the results of \cref{lemma:integral_inequality_frobenius_norm_phi,lemma:integral_inequality_tr_w} into $\tr(\Wbf^\top\Wbf(\tau))$ and $\norm{\Phibf^\top\Phibf(\tau)}_\frob$,
  we obtain the integral inequality:
  \begin{equation*}
    \begin{aligned}
      \norm{\Phibf^\top\Wbf^\top\Wbf\Phibf(t)}
      &\le \norm{\Phibf^\top\Wbf^\top\Wbf\Phibf(0)} + 4\rho\int_0^t\norm{\Phibf^\top\Wbf^\top\Wbf\Phibf(\tau)} \rd{\tau} \\
      &\phantom{\le} +\frac{16\rho te^{2\rho t} + (2\rho I_0 - 8)(e^{2\rho t} - 1)}{\rho^2}.
    \end{aligned}
  \end{equation*}
  This can be solved via \cref{theorem:gronwall_bellman}.
\end{proof}

\begin{lemma}
  \label{lemma:integral_inequality_l2_norm_wphix}
  For $\Wbf \in \sym_h$,
  for any $t$, the following bound holds:
  \begin{equation*}
    \begin{aligned}
      & \norm{\Phibf^\top\Wbf^\top\Wbf\Phibf(t)\xbf_0}_2^2 \\
      &\le \left\{ \norm{\Phibf^\top\Wbf^\top\Wbf\Phibf(0)\xbf_0}_2^2 + \Xi_1\norm{\xbf_0}_2^2 + \Xi_2\norm{\Phibf^\top\Phibf(0)\xbf_0}_2^2 \right\} \exp(2\rho t),
    \end{aligned}
  \end{equation*}
  where
  \begin{equation*}
    \begin{aligned}
      \Xi_1 &\defeq T_0^2\frac{e^{4\rho t}-1}{\rho} + 2T_0\frac{e^{4\rho t}(4\rho t - 1) + 1}{\rho^2} \\ &\phantom{\defeq} + 2\frac{e^{4\rho t}(8\rho^2 t^2 - 4\rho t + 1) - 1}{\rho^3} + 4\frac{e^{2\rho t}(2\rho t - 1) + 1}{\rho^2}, \\
      \Xi_2 &\defeq 2\frac{e^{2\rho t} - 1}{\rho},
    \end{aligned}
  \end{equation*}
  and $T_0 \defeq \tr(\Wbf^\top\Wbf(0))$.
\end{lemma}

\begin{proof}[Proof of \cref{lemma:integral_inequality_l2_norm_wphix}]
  By using the fundamental theorem of calculus, $\norm{\Phibf^\top\Wbf^\top\Wbf\Phibf(t)\xbf_0}_2^2$ is bounded as follows:
  \begin{equation}
    \begin{aligned}
      & \norm{\Phibf^\top\Wbf^\top\Wbf\Phibf(t)\xbf_0}_2^2 \\
      &\le \norm{\Phibf^\top\Wbf^\top\Wbf\Phibf(0)\xbf_0}_2^2 + 2\int_0^t\norm{\left(\diff{\Wbf\Phibf}{\tau}\right)^\top\Wbf\Phibf(\tau)\xbf_0}_2^2\rd{\tau} \\
      &\le \norm{\Phibf^\top\Wbf^\top\Wbf\Phibf(0)\xbf_0}_2^2 \\
        & \phantom{\le} + 2\underbrace{\int_0^t\norm{\dot\Phibf^\top\Wbf^\top\Wbf\Phibf(\tau)\xbf_0}_2^2\rd{\tau}}_\text{(A)} + 2\underbrace{\int_0^t\norm{\Phibf^\top\dot\Wbf^\top\Wbf\Phibf(\tau)\xbf_0}_2^2\rd{\tau}}_\text{(B)}.
    \end{aligned}
    \label{equation:proof:time_evolution_wphix_norm}
  \end{equation}
  To bound (A) in \cref{equation:proof:time_evolution_wphix_norm},
  we follow almost the same calculation as \cref{equation:proof:norm_dotphi_w_w_phi} in the proof of \cref{lemma:integral_inequality_l2_norm_phix} (therefore omitted) and obtain
  $\norm{\dot\Phibf^\top\Wbf^\top\Wbf\Phibf\xbf_0}_2^2 \le \tr(\Wbf^\top\Wbf)^2\norm{\xbf_0}_2^2$.
  To bound (B) in \cref{equation:proof:time_evolution_wphix_norm},
  we follow almost the same calculation as \cref{equation:proof:norm_phi_dotw_w_phi} in the proof of \cref{lemma:integral_inequality_l2_norm_phix} (therefore omitted) and obtain
  $\norm{\Phibf^\top\dot\Wbf^\top\Wbf\Phibf\xbf_0}_2^2 \le 2\norm{\Phibf^\top\Phibf\xbf_0}_2^2 + \rho\norm{\Phibf^\top\Wbf^\top\Wbf\Phibf\xbf_0}_2^2$.
  Here, the symmetry of $\Wbf$ is used.
  By substituting them back into (A) and (B) in \cref{equation:proof:time_evolution_wphix_norm}, we obtain the following bound:
  \begin{equation*}
    \begin{aligned}
      \norm{\Phibf^\top\Wbf^\top\Wbf\Phibf(t)\xbf_0}_2^2
      &\le \norm{\Phibf^\top\Wbf^\top\Wbf\Phibf(0)\xbf_0}_2^2 + 2\rho\int_0^t\norm{\Phibf^\top\Wbf^\top\Wbf\Phibf(\tau)\xbf_0}_2^2\rd{\tau} \\
        &\phantom{\le} + 4\norm{\xbf_0}_2^2\underbrace{\int_0^t\tr(\Wbf^\top\Wbf(\tau))^2\rd{\tau}}_{(\clubsuit)} + 4\underbrace{\int_0^t\norm{\Phibf^\top\Phibf(\tau)\xbf_0}_2^2\rd{\tau}}_{(\diamondsuit)}
      .
    \end{aligned}
  \end{equation*}
  The term ($\clubsuit$) can be evaluated by \cref{lemma:integral_inequality_tr_w} and integration by parts as follows:
  \begin{equation*}
    \begin{aligned}
      (\clubsuit)
      &\le \int_0^t(T_0 + 4\tau)^2\exp(4\rho\tau)\rd{\tau} \\
      &= \int_0^t(T_0^2 + 8T_0\tau + 16\tau^2)\exp(4\rho\tau)\rd{\tau} \\
      &= T_0^2\frac{e^{4\rho t}-1}{4\rho} + T_0\frac{e^{4\rho t}(4\rho t - 1) + 1}{2\rho^2} + \frac{e^{4\rho t}(8\rho^2 t^2 - 4\rho t + 1) - 1}{2\rho^3}, \\
    \end{aligned}
  \end{equation*}
  The term ($\diamondsuit$) can be evaluated by \cref{lemma:integral_inequality_l2_norm_phix} and integration by parts as follows:
  \begin{equation*}
    \begin{aligned}
      (\diamondsuit) &\le \int_0^t\left\{\norm{\Phibf^\top\Phibf(0)\xbf_0}_2^2 + 4\norm{\xbf_0}_2^2\tau\right\}e^{2\rho\tau}\rd{\tau} \\
      &= \norm{\Phibf^\top\Phibf(0)\xbf_0}_2^2\frac{e^{2\rho t} - 1}{2\rho} + \norm{\xbf_0}_2^2\frac{e^{2\rho t}(2\rho t - 1) + 1}{\rho^2}.
    \end{aligned}
  \end{equation*}
  Hence, we obtain the following integral inequality:
  \begin{equation*}
    \begin{aligned}
      \norm{\Phibf^\top\Wbf^\top\Wbf\Phibf(t)\xbf_0}_2^2
      &\le \norm{\Phibf^\top\Wbf^\top\Wbf\Phibf(0)\xbf_0}_2^2 + \Xi_1\norm{\xbf_0}_2^2 + \Xi_2\norm{\Phibf^\top\Phibf(0)\xbf_0}_2^2 \\
        &\phantom{\le} + 2\rho\int_0^t\norm{\Phibf^\top\Wbf^\top\Wbf\Phibf(\tau)\xbf_0}_2^2\rd{\tau},
    \end{aligned}
  \end{equation*}
  which can be solved by the Gr\"{o}nwall--Bellman inequality (\cref{theorem:gronwall_bellman}).
  As a result, the desired bound on $\norm{\Phibf^\top\Wbf^\top\Wbf\Phibf(t)\xbf_0}_2^2$ can be obtained.
\end{proof}

\section{Missing proofs}
\label{section:proofs}

\thmdynamics*

\begin{proof}[Proof of \cref{lemma:dynamics_matrix_expected}]
  \hypertarget{proof:dynamics_matrix_expected}{}
  To derive the $\Wbf$-dynamics, we begin with calculating the gradient $\nabla_\Wbf\losscos$.
  \begin{align*}
    &\!\!-\nabla_\Wbf\losscos \\
    &= \E \left[ \frac{1}{\norm{\Phibf\xbf'}_2} \frac{\norm{\Wbf\Phibf\xbf}_2 \nabla_\Wbf(\xbf^\top\Phibf^\top\Wbf^\top\Phibf\xbf') - \xbf^\top\Phibf^\top\Wbf^\top\Phibf\xbf' \nabla_\Wbf\norm{\Wbf\Phibf\xbf}_2}{\norm{\Wbf\Phibf\xbf}_2^2} \right] \\
    &= \E \left[ \frac{\nabla_\Wbf(\xbf^\top\Phibf^\top\Wbf^\top\zbf') - \omegabf^\top\zbf'\nabla_\Wbf\norm{\Wbf\Phibf\xbf}_2}{\norm{\Wbf\Phibf\xbf}_2} \right] \\
    &= \E \left[ \frac{\zbf'\xbf^\top\Phibf^\top - (\omegabf^\top\zbf')\frac{\Wbf\Phibf\xbf\xbf^\top\Phibf^\top}{\norm{\Wbf\Phibf\xbf}_2}}{\norm{\Wbf\Phibf\xbf}_2} \right] \\
    &= \E \left[ \zbf'\frac{\xbf^\top\Phibf^\top}{\norm{\Wbf\Phibf\xbf}_2} - (\omegabf^\top\zbf')\omegabf \frac{\xbf^\top\Phibf^\top}{\norm{\Wbf\Phibf\xbf}_2} \right].
  \end{align*}
  Here, $\Wbf$ follows the dynamics $\dot\Wbf = -\nabla_\Wbf\losscos - \rho\Wbf$,
  and hence we obtain $\dot\Wbf\Wbf^\top = \E[\zbf'\omegabf^\top - (\omegabf^\top\zbf')\omegabf\omegabf^\top] - \rho\Wbf\Wbf^\top$.

  To derive the $\Phibf$-dynamics, we calculate the gradient $\nabla_\Phibf\losscos$.
  \begin{align*}
    &\!\!-\nabla_\Phibf\losscos \\
    &= \E \left[ \frac{1}{\norm{\Phibf\xbf'}_2} \frac{\norm{\Wbf\Phibf\xbf}_2 \nabla_\Phibf(\xbf^\top\Phibf^\top\Wbf^\top\stopgrad(\Phibf)\xbf') - \xbf^\top\Phibf^\top\Wbf^\top\Phibf\xbf' \nabla_\Phibf\norm{\Wbf\Phibf\xbf}_2}{\norm{\Wbf\Phibf\xbf}_2^2} \right] \\
    &= \E \left[ \frac{1}{\norm{\Phibf\xbf'}_2} \frac{\norm{\Wbf\Phibf\xbf}_2 \Wbf^\top\Phibf\xbf'\xbf^\top - \xbf^\top\Phibf^\top\Wbf^\top\Phibf\xbf' \frac{\Wbf^\top\Wbf\Phibf\xbf\xbf^\top}{\norm{\Wbf\Phibf\xbf}_2}}{\norm{\Wbf\Phibf\xbf}_2^2} \right] \\
    &= \Wbf^\top \E \left[ \frac{\zbf'\xbf^\top - (\omegabf^\top\zbf')\omegabf\xbf^\top}{\norm{\Wbf\Phibf\xbf}_2} \right],
  \end{align*}
  from which $(-\nabla_\Phibf\losscos)\Phibf^\top\Wbf^\top = \Wbf^\top \E[\zbf'\omegabf^\top - (\omegabf^\top\zbf')\omegabf\omegabf^\top]$ follows.
  Thus, the dynamics $\dot\Phibf = -\nabla_\Phibf\losscos - \rho\Phibf$ can be written as
  $\dot\Phibf\Phibf^\top\Wbf^\top = \Wbf^\top \E[\zbf'\omegabf^\top - (\omegabf^\top\zbf')\omegabf\omegabf^\top] - \rho\Phibf\Phibf^\top\Wbf^\top$.
\end{proof}

\thmconci*

\begin{proof}[Proof of \cref{lemma:norm_concentration}]
  \hypertarget{proof:norm_concentration}{}
  We will show concentration of $\norm{\frac{1}{\sqrt{h\sigma^2}}\Phibf\xbf}_2^2$ and $\norm{\frac{1}{\sqrt{h^2\sigma^2}}\Wbf\Phibf\xbf}_2^2$.

  \paragraph*{Concentration of $\norm{\Phibf\xbf}_2^2$}:
  We begin with showing the first concentration.
  \begin{equation}
    \begin{aligned}
      \norm{\frac{1}{\sqrt{h\sigma^2}}\Phibf\xbf}_2^2
      &= \norm{\frac{1}{\sqrt{h}}\left(\Phibf\frac{\xbf - \xbf_0}{\sigma} + \Phibf\frac{\xbf_0}{\sigma}\right)}_2^2 \\
      &= \underbrace{\norm{\frac{1}{\sqrt{h}}\Phibf\frac{\xbf - \xbf_0}{\sigma}}_2^2}_\text{(A)} + 2\sigma^{-1}\underbrace{\inpr{\frac{1}{\sqrt{h}}\Phibf\frac{\xbf - \xbf_0}{\sigma}}{\frac{1}{\sqrt{h}}\Phibf\xbf_0}}_\text{(B)} + \norm{\frac{1}{\sqrt{h}}\Phibf\frac{\xbf_0}{\sigma}}_2^2.
    \end{aligned}
    \label{equation:proof:norm_decomposition}
  \end{equation}
  To deal with (A), which is a Gaussian chaos (namely, a quadratic form with standard normal vectors),
  we invoke the Hanson--Wright inequality \citep[Theorem 6.3.2]{Vershynin2018}.
  Note that $\frac{\xbf - \xbf_0}{\sigma}$ follows the standard normal distribution.
  Then, the following inequality holds with probability at least $1 - \delta$ (over the sampling of $\xbf$):
  \begin{equation}
    \abs{ \norm{\frac{1}{\sqrt{h}}\Phibf\frac{\xbf - \xbf_0}{\sigma}}_2 - \norm{\frac{1}{\sqrt{h}}\Phibf}_\frob } \le \sqrt{\frac{C_0\norm{\Phibf}^2\log\frac{2}{\delta}}{h}},
    \label{equation:proof:hanson_wright}
  \end{equation}
  where the expectation is taken over $\xbf \sim \Ncal(\xbf_0, \sigma^2\Ibf_d)$, and $C_0$ is an absolute constant irrelevant to $d$ and $h$.
  Now, we evaluate the deviation term and show it vanishes as $d, h \to \infty$.
  Since the deviation term contains $\norm{\Phibf}^2$ and it depends on the time $t$,
  we need to carefully evaluate its order in $d$ and $h$ along with time evolution.
  For this purpose, \cref{lemma:integral_inequality_spectral_norm_phi} is used to obtain
  $\norm{\Phibf(t)}^2 \le (\norm{\Phibf^\top\Phibf(0)} + 4t) \exp(2\rho t)$.
  Lastly, the Gaussian initialization of $\Phibf$ (\cref{assumption:initialization}) induces $\frac{1}{h}\norm{\Phibf^\top\Phibf(0)} = \stocho(1)$ (by \cref{lemma:spectral_norm_init}).
  Thus, the deviation term of \cref{equation:proof:hanson_wright} is bounded from above as follows:
  \begin{equation*}
    \sqrt{\frac{C_0(\norm{\Phibf^\top\Phibf(0)} + 4t)\exp(2\rho t)\log\frac{2}{\delta}}{h}}
    = \stocho(1),
  \end{equation*}
  from which we conclude as follows:
  \begin{equation*}
    \norm{\frac{1}{\sqrt{h}}\Phibf\frac{\xbf - \xbf_0}{\sigma}}_2^2
    = \norm{\frac{1}{\sqrt{h}}\Phibf}_\frob^2 + \stocho(1).
  \end{equation*}

  Next, we deal with (B) in \cref{equation:proof:norm_decomposition}.
  The term (B) is equivalent to $\inpr{\frac{1}{h}\Phibf^\top\Phibf\xbf_0}{\frac{\xbf - \xbf_0}{\sigma}}$,
  which is a linear combination of the standard normal random variables.
  Its concentration (to mean $0$) can be established by the general Hoeffding's inequality \citep[Theorem 2.6.3]{Vershynin2018} as follows:
  With probability at least $1 - \delta$ (over the sampling of $\xbf$),
  \begin{equation}
    \text{(B)} = 
    \abs{\inpr{\frac{1}{h}\Phibf^\top\Phibf\xbf_0}{\frac{\xbf - \xbf_0}{\sigma}}} \le \sqrt{\frac{C_1\norm{\Phibf^\top\Phibf\xbf_0}_2^2\log\frac{2}{\delta}}{h^2}},
    \label{equation:proof:deviation_cross_term}
  \end{equation}
  where $C_1$ is an absolute constant irrelevant to $d$ and $h$.
  We need to evaluate $\norm{\Phibf^\top\Phibf(t)\xbf_0}_2^2$ by noting its time dependency again.
  For this purpose, \cref{lemma:integral_inequality_l2_norm_phix} is used to obtain
  $\norm{\Phibf^\top\Phibf(t)\xbf_0}_2^2 \le (\norm{\Phibf^\top\Phibf(0)\xbf_0}_2^2 + 4\norm{\xbf_0}_2^2t)\exp(2\rho t)$.
  Here, $\frac{1}{h^2}\norm{\Phibf^\top\Phibf(0)\xbf_0}_2^2 = \stocho(1)$ (\cref{lemma:l2_norm_init}) holds.
  In addition, $\xbf_0 \sim \Ncal(\zerobf, \Ibf)$ (\cref{assumption:standard_mvn}) indicates that $\norm{\xbf_0}_2^2$ is the sum of independent zero-mean sub-exponential random variables,
  from which Bernstein's inequality claim $\norm{\xbf_0}_2^2 = \stochO(d)$ \citep[Corollary 2.8.3]{Vershynin2018}.
  Plugging them into the upper bound of $\norm{\Phibf^\top\Phibf(t)\xbf_0}_2^2$, we deduce
  \begin{equation*}
    \begin{aligned}
      \text{(B)}
      &\le \sqrt{C_1 \log\frac{2}{\delta} \left(\frac{\norm{\Phibf^\top\Phibf(0)\xbf_0}_2^2}{h^2} + 4t\frac{\norm{\xbf_0}_2^2}{h^2}\right)e^{2\rho t}} \\
      &= \sqrt{\stocho(1) + \stochO(\alpha h^{-1})}
      = \stocho(1).
    \end{aligned}
  \end{equation*}

  Eventually, the concentration of (A) and (B) is established and the conclusion follows from \cref{equation:proof:norm_decomposition}.

  \paragraph*{Concentration of $\norm{\Wbf\Phibf\xbf}_2^2$}: 
  In the same manner as \cref{equation:proof:norm_decomposition}, we have the following decomposition:
  \begin{equation}
    \begin{aligned}
      \norm{\frac{1}{\sqrt{h^2\sigma^2}}\Wbf\Phibf\xbf}_2^2
      &= \norm{\frac{1}{h}\Wbf\Phibf\frac{\xbf-\xbf_0}{\sigma}}_2^2 + \frac{2}{\sigma}\inpr{\frac{1}{h}\Wbf\Phibf\frac{\xbf-\xbf_0}{\sigma}}{\frac{1}{h}\Wbf\Phibf\xbf_0} \\&\phantom{=} + \norm{\frac{1}{h}\Wbf\Phibf\frac{\xbf_0}{\sigma}}_2^2.
    \end{aligned}
    \label{equation:proof:norm_decomposition_2}
  \end{equation}
  The subsequent analysis follows in a very similar way to the analysis of $\norm{\frac{1}{\sqrt{h\sigma^2}}\Phibf\xbf}_2^2$.
  Indeed, we can obtain the following inequalities (each of them with probability at least $1-\delta$, respectively):
  \begin{align}
    \abs{\norm{\frac{1}{h}\Wbf\Phibf\frac{\xbf-\xbf_0}{\sigma}}_2 - \norm{\frac{1}{h}\Wbf\Phibf}_\frob} & \le \sqrt{\frac{C_2\norm{\Wbf\Phibf}^2\log\frac{2}{\delta}}{h^2}}, \label{equation:proof:hanson_wright_2} \\
    \abs{\inpr{\frac{1}{h}\Wbf\Phibf\frac{\xbf-\xbf_0}{\sigma}}{\frac{1}{h}\Wbf\Phibf\xbf_0}} & \le \sqrt{\frac{C_3\norm{\Phibf^\top\Wbf^\top\Wbf\Phibf\xbf_0}_2^2\log\frac{2}{\delta}}{h^4}}, \label{equation:proof:deviation_cross_term_2}
  \end{align}
  where $C_2$ and $C_3$ are absolute constants (see \cref{equation:proof:hanson_wright,equation:proof:deviation_cross_term}).

  To deal with \cref{equation:proof:hanson_wright_2},
  we control the spectral norm $\norm{\Wbf\Phibf(t)}$ along time evolution by using \cref{lemma:integral_inequality_spectral_norm_phiw},
  and obtain the following bound:
  \begin{equation*}
    \norm{\Wbf\Phibf(t)}^2
    \le \left\{ \norm{\Phibf^\top\Wbf^\top\Wbf\Phibf(0)} + \frac{16\rho te^{2\rho t} + (2\rho I_0 - 8)(e^{2\rho t} - 1)}{\rho^2} \right\}e^{4\rho t},
  \end{equation*}
  where $I_0 \defeq \tr(\Wbf^\top\Wbf(0)) + \norm{\Phibf^\top\Phibf(0)}_\frob$.
  By plugging this bound back into \cref{equation:proof:hanson_wright_2} and using \cref{lemma:spectral_norm_init,lemma:frobenius_norm_init}, we obtain
  \begin{equation*}
    \norm{\frac{1}{h}\Wbf\Phibf\frac{\xbf-\xbf_0}{\sigma}}_2^2 = \norm{\frac{1}{h}\Wbf\Phibf}_\frob^2 + \stocho(1).
  \end{equation*}

  Next, we deal with \cref{equation:proof:deviation_cross_term_2} by controlling the L2 norm $\norm{\Phibf^\top\Wbf^\top\Wbf\Phibf(t)\xbf_0}_2^2$ along time evolution.
  By using \cref{lemma:integral_inequality_l2_norm_wphix}, we obtain the following bound:
  \begin{equation*}
    \begin{aligned}
      &\norm{\Phibf^\top\Wbf^\top\Wbf\Phibf(t)\xbf_0}_2^2 \\
      &\le \left\{ \norm{\Phibf^\top\Wbf^\top\Wbf\Phibf(0)\xbf_0}_2^2 + \Ocal(\norm{\Phibf^\top\Phibf(0)\xbf_0}_2^2) + \norm{\xbf_0}_2^2 \Ocal(\tr(\Wbf^\top\Wbf(0))^2) \right\} e^{2\rho t},
    \end{aligned}
  \end{equation*}
  where the order term $\Ocal(\tr(\Wbf^\top\Wbf(0))^2)$ hides the dependency on $t$.
  We now combine \cref{lemma:frobenius_norm_init,lemma:l2_norm_init} and the consequence of Bernstein's inequality $\norm{\xbf_0}_2^2 = \stochO(d)$
  and substitute them into \cref{equation:proof:deviation_cross_term_2}.
  Then, we obtain 
  \begin{equation*}
    \begin{aligned}
      &\abs{\inpr{\frac{1}{h}\Wbf\Phibf\frac{\xbf-\xbf_0}{\sigma}}{\frac{1}{h}\Wbf\Phibf\xbf_0}} \\
      &\le \sqrt{C_3'\left\{\frac{\norm{\Phibf^\top\Wbf^\top\Wbf\Phibf(0)\xbf_0}_2^2}{h^4} + \frac{\Ocal(\norm{\Phibf^\top\Phibf(0)\xbf_0}_2^2)}{h^4} + \frac{\norm{\xbf_0}_2^2\Ocal(\tr(\Wbf^\top\Wbf(0))^2)}{h^4}\right\}} \\
      &= \sqrt{\stocho(1) + \stocho(1) \cdot h^{-2} + \stochO(d) \cdot \stocho(1) \cdot h^{-2}} \\
      &= \stocho(1),
    \end{aligned}
  \end{equation*}
  where $C_3' \defeq C_3e^{2\rho t}\log\tfrac{2}{\delta}$.

  Hence, the concentration result for $\norm{\frac{1}{\sqrt{h^2\sigma^2}}\Wbf\Phibf\xbf}_2^2$ is established
  by substituting \cref{equation:proof:hanson_wright_2,equation:proof:deviation_cross_term_2} back into \cref{equation:proof:norm_decomposition_2}.
\end{proof}

\thmconcii*

\begin{proof}[Proof of \cref{lemma:norm_concentration_x}]
  \hypertarget{proof:norm_concentration_x}{}
  To establish concentration of $\norm{\Phibf\xbf_0}_2$,
  we invoke the Hanson--Wright inequality \citep[Theorem 6.3.2]{Vershynin2018}:
  For an absolute constant $C_0$,
  \begin{equation*}
    \abs{\norm{\frac{1}{\sqrt{h\sigma^2}}\Phibf\xbf_0}_2 - \norm{\frac{1}{\sqrt{h\sigma^2}}\Phibf}_\frob} \le \sqrt{\frac{C_0\norm{\Phibf}^2\log\frac{2}{\delta}}{h\sigma^2}},
  \end{equation*}
  with probability at least $1 - \delta$.
  Here, we further derive the upper bound of the right-hand side by \cref{lemma:integral_inequality_spectral_norm_phi}:
  \begin{equation*}
    \frac{\norm{\Phibf(t)}^2}{h}
    \le \frac{(\norm{\Phibf^\top\Phibf(0)} + 4t)\exp(2\rho t)}{h}
    = \stocho(1),
  \end{equation*}
  where the last identity follows from \cref{lemma:spectral_norm_init}.
  Thus, the concentration of $\norm{\Phibf\xbf_0}_2$ is shown.

  To establish concentration of $\norm{\Wbf\Phibf\xbf_0}$,
  we invoke the Hanson--Wright inequality again:
  For an absolute constant $C_1$,
  \begin{equation*}
    \abs{\norm{\frac{1}{\sqrt{h^2\sigma^2}}\Wbf\Phibf\xbf_0} - \norm{\frac{1}{\sqrt{h^2\sigma^2}}\Wbf\Phibf}_\frob} \le \sqrt{\frac{C_1\norm{\Wbf\Phibf}^2\log\frac{2}{\delta}}{h^2\sigma^2}},
  \end{equation*}
  with probability at least $1 - \delta$.
  We can show $\frac{1}{h^2}\norm{\Wbf\Phibf(t)}^2 = \stocho(1)$ in the same way as in the proof of \cref{lemma:norm_concentration}.
\end{proof}

\thmevalexpt*

\begin{proof}[Proof of \cref{lemma:evaluation_expectation}]
  \hypertarget{proof:evaluation_expectation}{}
  To evaluate $\Hbf = \E[\zbf'\omegabf^\top - (\omegabf^\top\zbf')\omegabf\omegabf^\top] \defeq \Hbf_1 - \Hbf_2$,
  where $\Hbf_1 \defeq \E[\zbf'\omegabf^\top]$ and $\Hbf_2  = \E[(\omegabf^\top\zbf')\omegabf\omegabf^\top]$,
  we evaluate the normalizers $\norm{\Phibf\xbf'}_2^{-1}$ and $\norm{\Wbf\Phibf\xbf}_2^{-1}$ first.
  By \cref{lemma:norm_concentration,lemma:norm_concentration_x},
  \begin{equation*}
    \begin{aligned}
      \frac{1}{\norm{\Phibf\xbf'}_2}
      &= \frac{1}{\sqrt{h\sigma^2}} \cdot \left\{\norm{\frac{1}{\sqrt{h}}\Phibf}_\frob^2 + \norm{\frac{1}{\sqrt{h\sigma^2}}\Phibf}_\frob^2 + \stocho(1)\right\}^{-1/2} \\
      &= \frac{1}{\sqrt{h\sigma^2}} \cdot \frac{1}{\sqrt{1+\sigma^{-2}}\norm{\frac{1}{\sqrt{h}}\Phibf}_\frob + \stocho(1)} \\
      &\stackrel{(\clubsuit)}= \frac{1}{\sqrt{h\sigma^2}} \cdot \left\{ \frac{1}{\sqrt{1+\sigma^{-2}}\norm{\frac{1}{\sqrt{h}}\Phibf}_\frob} + \stocho(1) \right\} \\
      &= \frac{1}{\sqrt{1+\sigma^2}} \cdot \frac{1}{\norm{\Phibf}_\frob} + \stocho(1),
    \end{aligned}
  \end{equation*}
  where $(\clubsuit)$ is due to the first-order Taylor expansion $f(\epsilon) = \frac{1}{x + \epsilon} \approx \frac{1}{x} - \frac{\epsilon}{x^2}$ around $\epsilon = 0$.
  Similarly, we have
  \begin{equation*}
    \frac{1}{\norm{\Wbf\Phibf\xbf}_2} = \frac{1}{\norm{\Psibf\xbf}_2} = \frac{1}{\sqrt{1+\sigma^2}} \cdot \frac{1}{\norm{\Psibf}_\frob} + \stocho(1).
  \end{equation*}

  Next, we evaluate $\Hbf_1$.
  \begin{equation*}
    \begin{aligned}
      \Hbf_1
      &= \E_{\xbf_0} \E_{\xbf, \xbf'} \left[ \frac{\Phibf\xbf'}{\norm{\Phibf\xbf'}_2} \left(\frac{\Psibf\xbf}{\norm{\Psibf\xbf}_2}\right)^\top \right] \\
      &= \E_{\xbf_0} \left[ \frac{1}{(1+\sigma^2) \norm{\Phibf}_\frob \norm{\Psibf}_\frob} \E_{\xbf, \xbf'}[\Phibf\xbf'\xbf^\top\Psibf^\top] \right] + \stocho(1) \\
      &= \frac{1}{1+\sigma^2} \frac{\Phibf}{\norm{\Phibf}_\frob} \frac{\Psibf^\top}{\norm{\Psibf}_\frob} + \stocho(1),
    \end{aligned}
  \end{equation*}
  where we used $\E_{\xbf_0} \E_{\xbf, \xbf'} [\xbf'\xbf^\top] = \E_{\xbf_0} [\xbf_0\xbf_0^\top] = \Ibf_d$ at the last identity.
  We can evaluate $\Hbf_2$ similarly.
  \begin{equation*}
    \begin{aligned}
      \Hbf_2
      &= \E_{\xbf_0} \E_{\xbf, \xbf'} \left[ \frac{(\xbf^\top\Psibf^\top\Phibf\xbf') \Psibf\xbf\xbf^\top\Psibf^\top}{\norm{\Phibf\xbf'}_2 \norm{\Psibf\xbf}_2^3} \right] \\
      &= \E_{\xbf_0} \left[ \frac{1}{(1+\sigma^2)^2 \norm{\Phibf}_\frob \norm{\Psibf}_\frob^3} \Psibf \E_{\xbf, \xbf'}[(\xbf^\top\Psibf^\top\Phibf\xbf')\xbf\xbf^\top] \Psibf^\top \right] + \stocho(1),
    \end{aligned}
  \end{equation*}
  where the inner expectation $\E[(\xbf^\top\Psibf^\top\Phibf\xbf')\xbf\xbf^\top]$ requires the moment evaluations of Gaussian:
  \begin{equation*}
    \begin{aligned}
      &\E_{\xbf_0} \E_{\xbf, \xbf'}[(\xbf^\top\Psibf^\top\Phibf\xbf')\xbf\xbf^\top] \\
      &= \E_{\xbf \mid \xbf_0}[\xbf\xbf^\top\Abf\xbf_0\xbf^\top]
        && \text{$\triangleleft$ $\Abf \defeq \Psibf^\top\Phibf$} \\
      &= \sigma^2\E[\Abf\xbf_0\xbf_0^\top] + \sigma^2\E[\xbf_0\xbf_0^\top\Abf] \\
        &\phantom{=} + \E[\xbf_0\xbf_0^\top\Abf\xbf_0\xbf_0^\top] + \sigma^2\E[\xbf^\top\Abf\xbf_0]\Ibf_d
        && \text{$\triangleleft$ \cite[\S8.2.3]{Petersen2012}} \\
      &= 2\sigma^2\Abf + \E[\xbf_0\xbf_0^\top\Abf\xbf_0\xbf_0^\top] + \sigma^2\tr(\Abf)\Ibf_d
        && \text{$\triangleleft$ \cite[\S8.2.2]{Petersen2012}} \\
      &= 2\sigma^2\Abf + \{2\Abf + \tr(\Abf)\Ibf_d\} + \sigma^2\tr(\Abf)\Ibf_d
        && \text{$\triangleleft$ \cite[\S8.2.4]{Petersen2012}} \\
      &= (1 + \sigma^2)\{2\Psibf^\top\Phibf + \tr(\Psibf^\top\Phibf)\Ibf_d\}.
    \end{aligned}
  \end{equation*}
  Note that $\Psibf^\top\Phibf = \Abf = \Abf^\top = \Phibf^\top\Psibf$ under \cref{assumption:w_symmetric}.
  By plugging this back,
  \begin{equation*}
    \begin{aligned}
      \Hbf_2
      &= \frac{1}{1+\sigma^2} \left\{ 2\tilde\Psibf\tilde\Phibf^\top\tilde\Psibf\tilde\Psibf^\top + \tr(\tilde\Psibf^\top\tilde\Phibf)\tilde\Psibf\tilde\Psibf^\top \right\} + \stocho(1).
    \end{aligned}
  \end{equation*}
  The desired expression of $\Hbf = \Hbf_1 - \Hbf_2$ is thereby obtained.
\end{proof}

\thmcommutator*

In the proof, we leverage the elementary properties of commutators.
\begin{lemma}
  \label{lemma:commutator}
  For matrices $\Abf$, $\Bbf$, and $\Cbf$ with the same size, we have the following identities.
  \begin{enumerate}
    \item $[\Abf, \Abf] = \Obf$.
    \item $[\Abf, \Bbf] = -[\Bbf, \Abf]$.
    \item $[\Abf, \Bbf\Cbf] = [\Abf, \Bbf]\Cbf + \Bbf[\Abf, \Cbf]$.
    \item $[\Abf\Bbf, \Cbf] = \Abf[\Bbf, \Cbf] + [\Abf, \Cbf]\Bbf$.
  \end{enumerate}
\end{lemma}

\begin{proof}[Proof of \cref{proposition:commutator_dynamics}]
  \hypertarget{proof:commutator_dynamics}{}
  First, compute the time derivative $\dot\Lbf = \Fbf\dot\Wbf - \dot\Wbf\Fbf + \dot\Fbf\Wbf - \Wbf\dot\Fbf$:
  \begin{equation*}
    \begin{aligned}
      \Fbf\dot\Wbf - \dot\Wbf\Fbf
      &= \Fbf\Hbf^\top\Wbf^{-1} - \Wbf^{-1}\Hbf\Fbf - \rho\Lbf, \\
      \dot\Fbf\Wbf - \Wbf\dot\Fbf
      &= \Wbf\Hbf - \Hbf^\top\Wbf + \Wbf^{-1}\Hbf^\top\Wbf^2 - \Wbf^2\Hbf\Wbf^{-1} - 2\rho\Lbf,
    \end{aligned}
  \end{equation*}
  which implies
  \begin{equation}
    \label{equation:proof:commutator_derivative}
    \dot\Lbf = (\Fbf\Hbf^\top\Wbf^{-1} - \Wbf^{-1}\Hbf\Fbf) + (\Wbf\Hbf - \Hbf^\top\Wbf) + (\Wbf^{-1}\Hbf^\top\Wbf^2 - \Wbf^2\Hbf\Wbf^{-1}) - 3\rho\Lbf.
  \end{equation}
  We substitute $\Hbf = \hat\Hbf$. Then,
  \begin{equation*}
    \Wbf\Hbf - \Hbf^\top\Wbf
    = -2\frac{\Wbf^2\Fbf\Wbf\Fbf\Wbf - \Wbf\Fbf\Wbf\Fbf\Wbf^2}{(1+\sigma^2)N_\Phi N_\Psi^3} - N_\times\frac{\Wbf^2\Fbf\Wbf - \Wbf\Fbf\Wbf^2}{(1+\sigma^2)N_\Phi N_\Psi},
  \end{equation*}
  which can be simplified by \cref{lemma:commutator} as follows:
  \begin{equation*}
    \begin{cases}
      \Wbf^2\Fbf\Wbf\Fbf\Wbf - \Wbf\Fbf\Wbf\Fbf\Wbf^2
      = [\Wbf, \Wbf\Fbf\Wbf\Fbf]\Wbf
      = -(\Lbf\Wbf\Fbf + \Fbf\Wbf\Lbf)\Wbf, \\
      \Wbf^2\Fbf\Wbf - \Wbf\Fbf\Wbf^2
      = [\Wbf, \Wbf\Fbf\Wbf]
      = -\Wbf\Lbf\Wbf.
    \end{cases}
  \end{equation*}
  With the same technique, \cref{equation:proof:commutator_derivative} can be simplified as follows:
  \begin{equation*}
    \begin{aligned}
      \dot\Lbf
      &= \frac{(\Lbf\Wbf + \Wbf\Lbf) - (\Fbf\Lbf\Wbf^{-1} + \Wbf^{-1}\Lbf\Fbf)}{(1+\sigma^2)N_\Phi N_\Psi} \\
      &\phantom{=} - 2\frac{(\Wbf\Fbf\Wbf\Lbf\Wbf + \Wbf\Lbf\Wbf\Fbf\Wbf) + \Wbf^2(\Fbf\Wbf\Lbf + \Lbf\Wbf\Fbf)}{(1+\sigma^2)N_\Phi N_\Psi^3} \\
      &\phantom{=} -N_\times\frac{\Lbf\Wbf^2 + \Wbf^2\Lbf}{(1+\sigma^2)N_\Phi N_\Psi} - 3\rho\Lbf.
    \end{aligned}
  \end{equation*}
  By using $\vec{\Abf\Lbf\Bbf + \Bbf\Lbf\Abf} = (\Bbf \otimes \Abf + \Abf \otimes \Bbf)\vec{\Lbf} = (\Abf \oplus \Bbf) \vec{L}$ for $\Abf, \Bbf \in \sym_d$,
  we obtain $\diff{\vec{\Lbf}}{t} = -\Kbf\vec{\Lbf}$.

  Finally, by applying \citet[Lemma 2]{Tian2021ICML},
  the dynamics of $\Lbf(t)$ satisfies $\norm{\vec{\Lbf(t)}}_2 \le e^{-2\lambda_0t}\norm{\vec{\Lbf(0)}}_2 \to 0$
  under the assumption $\inf_{t \ge 0} \eigmin{(\Kbf(t))} \ge \lambda_0 > 0$.
\end{proof}

\thmeigenspace*

\begin{proof}[Proof of \cref{proposition:stable_eigenspace}]
  \hypertarget{proof:stable_eigenspace}{}
  The proof mostly follows the discussion of \citet[Appendix B.1]{Tian2021ICML}.
  To apply their discussion, all we need to check is the existence of diagonal matrices $\Gbf_1$ and $\Gbf_2$
  such that $\dot\Wbf = \Ubf\Gbf_1\Ubf^\top$ and $\dot\Fbf = \Ubf\Gbf_2\Ubf^\top$
  under the dynamics \cref{equation:dynamics_matrix_expected} with $\Hbf = \hat\Hbf$.

  For $\dot\Wbf$, invertibility of $\Wbf$ implies $\dot\Wbf = \Wbf^{-1}\hat\Hbf - \rho\Wbf$ from the dynamics \cref{equation:dynamics_matrix_expected}.
  With simultaneous diagonalization $\Wbf = \Ubf\Lambdabf_W\Ubf^\top$ and $\Fbf = \Ubf\Lambdabf_F\Ubf^\top$,
  we have $\Wbf^{-1} = \Ubf\Lambdabf_W^{-1}\Ubf^\top$ and $\hat\Hbf = \Ubf\Lambdabf_{\hat H}\Ubf^\top$ for some diagonal matrix $\Lambdabf_{\hat H}$.
  Hence, $\dot\Wbf = \Ubf\Gbf_1\Ubf^\top$ for some diagonal matrix $\Gbf_1$.

  In the same manner, we can verify $\dot\Fbf = \Ubf\Gbf_2\Ubf^\top$ for some diagonal matrix $\Gbf_2$.
\end{proof}

\section{Analysis of saddle-node bifurcation}
\label{section:regime_shift}
\begin{figure}[t]
  \centering
  \input{figure/root_finding}
  \caption{
    Plots of $g(x) = -Ax^6 + x^2$ (blue) and $h(x) = -Bx$ (red).
    \textbf{(Left)} $(A,B) = (1.5,0.6)$
    \textbf{(Center)} $(A,B) = (1.5,0.4)$
    \textbf{(Right)} $(A,B) = (1.5,0)$
  }
  \label{figure:root_finding}
\end{figure}
In \cref{section:regimes}, we claimed that the $w_j$-dynamics \eqref{equation:dynamics_invariant_parabola} entails the three regimes, mainly based on categorization of the numerical plots with different values of $(N_\Phi, N_\Psi, \rho)$ in \cref{figure:stability}.
Here, we show that the equilibrium point sets with different parameter values can indeed be classified into the three regimes.

First, we need slight approximation because the $w_j$-dynamics \eqref{equation:dynamics_invariant_parabola} is sixth-order and extremely challenging to deal with analytically in general.
We choose to set $N_\times (= \tr(\tilde\Phibf^\top\tilde\Psibf)) \approx 0$.
This can be confirmed in our simple numerical experiments in \cref{figure:experiment}.
Then, the $w_j$-dynamics reads:
\begin{equation*}
  \dot{w_j} \approx \frac{1}{(1+\sigma^2)N_\Phi N_\Psi} \underbrace{\left\{ -\frac{2}{N_\Psi^2}w_j^6 + w_j^2 - \rho(1+\sigma^2)N_\Phi N_\Psi w_j \right\}}_{= f(w_j)}.
\end{equation*}
Let us write $f(x) = -Ax^6 + x^2 - Bx$ with $A \defeq 2/N_\Psi^2 > 0$ and $B \defeq \rho(1+\sigma^2)N_\Phi N_\Psi \ge 0$.
Now, we focus on finding the roots of $f(x) = 0$, which are the equilibrium points of the $w_j$-dynamics.
In \cref{figure:root_finding}, we show the graphs of $g(x) = -Ax^6 + x^2$ and $h(x) = -Bx$ with different $B$.
When $B = 0$, we can analytically find the roots of $f(x) = g(x) = 0$ by $g(x) = -Ax^2(x^2+A^{-1/2})(x+A^{-1/4})(x-A^{-1/4})$ and $x = 0, \pm A^{-1/4}$.
This corresponds to the Stable regime in \cref{figure:regime}.
When $B$ is larger than zero and as $h(x) = -Bx$ tilts towards negative slightly, we have four roots as seen in \cref{figure:root_finding} (Center).
This corresponds to the Acute regime in \cref{figure:regime}.
Finally, when $B$ is significantly larger than zero, we have only two roots as seen in \cref{figure:root_finding} (Left), which corresponds to the Collapse regime in \cref{figure:regime}.
These three cases are interpolated smoothly as $B \propto \rho N_\Phi N_\Psi$ changes; to put it differently, as regularization strength $\rho$ and norms $N_\Phi,N_\Psi$ decrease, the regime approaches the Stable.
Note again that we will never perfectly attain the Stable regime because the $w_j$-dynamics diverges as $N_\Phi, N_\Psi \to 0$.

\section{Additional numerical experiments}
\label{section:additional_experiments}

\subsection{Full detail of linear encoder setup}

\begin{figure}[t]
  \centering
  \includegraphics[width=0.7\textwidth]{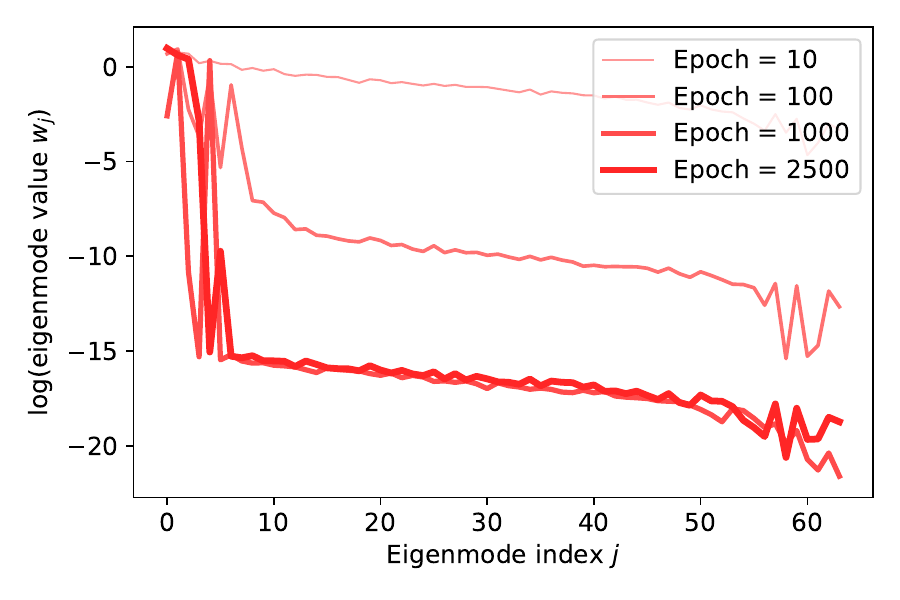}
  \caption{
    Time evolution of the eigenvalues.
    At each epoch, the projection head eigenvalue $w_j$ for each $j \in \set{1, 2, \dots, 64}$ is plotted.
    The eigenvalue values are uniformly averaged within $[\mathrm{epoch}-50, \mathrm{epoch}+50]$ to avoid visual clutter due to eigenvalue fluctuation.
  }
  \label{figure:eigenvalue_histogram}
\end{figure}

\begin{figure}[t]
  \centering
  \includegraphics[width=0.40\textwidth]{figure/experiment_gaussian_regime_1}
  \includegraphics[width=0.40\textwidth]{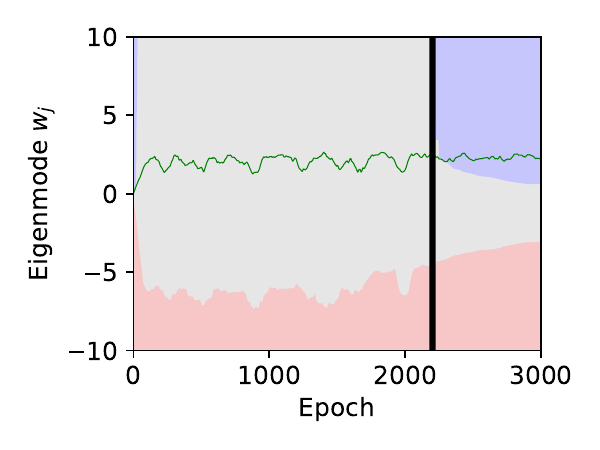}
  \includegraphics[width=0.40\textwidth]{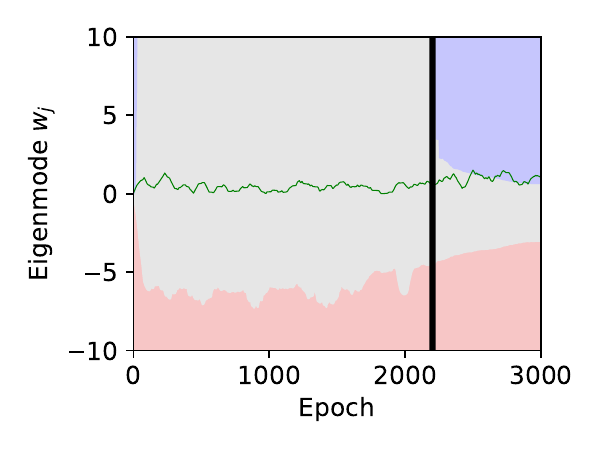}
  \includegraphics[width=0.40\textwidth]{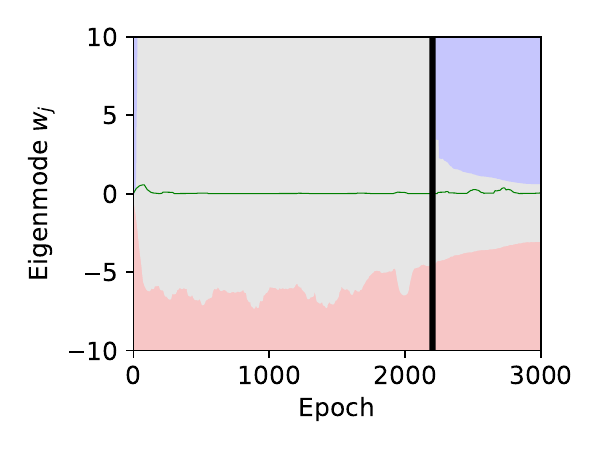}
  \caption{
    The eigenvalues of the projection head $w_j$ are plotted, with background colors illustrating three intervals where \colorbox{red!20}{\raisebox{0pt}[0.8ex]{$w_j$ diverges}}, \colorbox{black!20}{\raisebox{0pt}[0.8ex]{$w_j$ collapses}}, and \colorbox{blue!20}{\raisebox{0pt}[0.8ex]{$w_j$ stably converges}} at each epoch. Each color corresponds to those in \cref{figure:regime}. The vertical black line indicates the bifurcates from Collapse (epoch < 2200) to Acute (epoch > 2200).
    The eigenvalue values are uniformly averaged within $[\mathrm{epoch}-50, \mathrm{epoch}+50]$ to avoid visual clutter due to eigenvalue fluctuation.
    \textbf{(Top left)} $j=1$ (the largest eigenvalue);
    \textbf{(Top center)} $j=2$ (the second largest eigenvalue);
    \textbf{(Top center)} $j=3$ (the third largest eigenvalue);
    \textbf{(Bottom left)} $j=4$ (the fourth largest eigenvalue);
  }
  \label{figure:each_eigenvalue}
\end{figure}

We further analyze the numerical experiments in \cref{section:experiments}.
In \cref{section:experiments}, we focused on illustration of the leading eigenvalue $w_j$ of the projection head, which is shown in \cref{figure:experiment}.
Here, we investigate the other eigenvalues.
Throughout the analysis, we focus on the absolute value of the eigenvalue $|w_j|$ because the eigendecomposition is non-unique;
indeed, due to the decomposition $\Wbf = \sum_{j=1}^{64}w_j\ubf_j\ubf_j^\top$ ($\ubf_j$ is the eigenvector), the eigenvalue signs are irrelevant to the norms of the eigenvectors.
Flipping the sign does not affect the orthonormality of the eigenvectors, keeping $\Ubf$ to be a orthogonal matrix.
After taking the absolute values, all eigenvalues are sorted in the descending order, where $j=1$ and $j=64$ correspond to the largest and smallest, respectively.

\Cref{figure:eigenvalue_histogram} illustrates time evolution of the eigenvalue values of the projection head.
Initially ($\mathrm{epoch} = 10$), the eigenvalue distribution mildly concentrates around the origin, which can be seen in the initialization of the eigenvalue distribution in \cref{figure:eigeninit} as well.
As time evolves, the distribution quickly concentrates at zero very sharply, whereas a few positive eigenvalues that are significantly larger than zero remains.

Next, we investigate time evolution of each eigenvalue individually.
\Cref{figure:each_eigenvalue} shows time evolution of the largest ($j=1$), second largest ($j=2$), third largest ($j=3$), and fourth largest ($j=4$), using the same illustration as \cref{figure:experiment}.
The top left figure ($j=1$) is the same one as in \cref{figure:experiment}.
As can be seen in this case, only $w_{1}$, $w_{2}$, and $w_{3}$ remain positive and all the other eigenvalues (including $5 \le j \le 64$ omitted from \cref{figure:each_eigenvalue}) converges to nearly zero.
In our theoretical analysis, we argued that there are only two stable equilibrium in the Acute regime ($w_j = 0$ and $w_j = w_\blacktriangledown^{(+)}$ in \cref{figure:regime}).
Given this, the convergences of $w_{\set{1,2,3}}$ to positive values (that even fall in the stable interval) and $\set{w_j}_{j=3}^{63}$ to zero are reasonable in terms of the dynamics.
Moreover, this convergence avoids the complete collapse $\Wbf \to \Obf$; the complete collapse is avoided if several (but not necessarily all) eigenvalues remain to be non-zero.

\subsection{Simulation with nonlinear encoder}
\label{section:experiments_nonlinear}

\begin{figure}[t]
  \centering
  \includegraphics[width=0.7\textwidth]{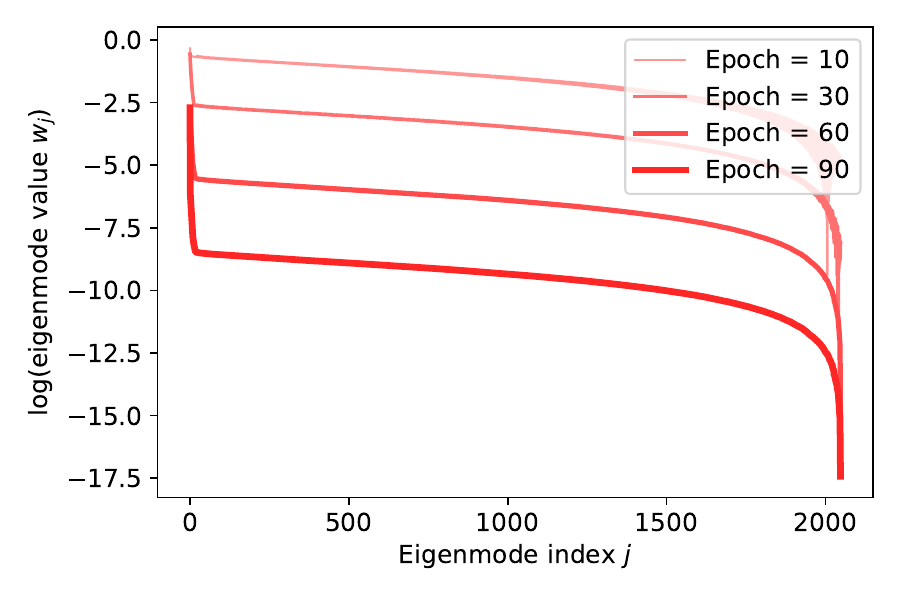}
  \caption{
    Time evolution of the eigenvalues (trained with the nonlinear encoder).
    At each epoch, the projection head eigenvalue $w_j$ for each $j \in \set{1, 2, \dots, 2048}$ is plotted.
  }
  \label{figure:eigenvalue_histogram_cifar}
\end{figure}

\begin{figure}[t]
  \centering
  \includegraphics[width=0.40\textwidth]{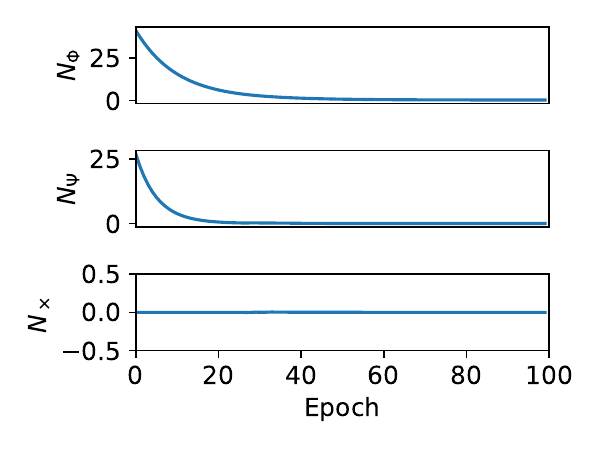}
  \includegraphics[width=0.40\textwidth]{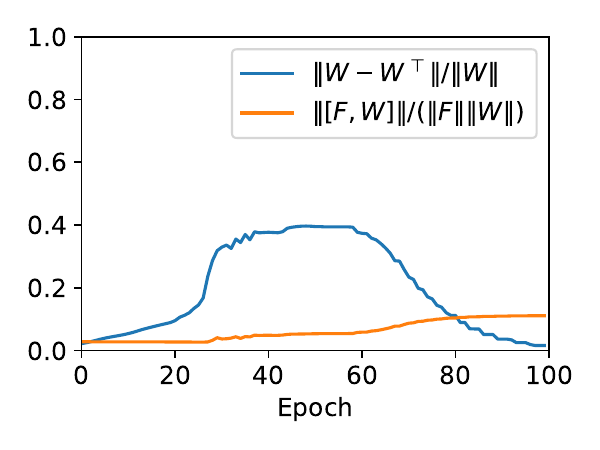}
  \caption{
    Numerical simulation of the SimSiam model with the nonlinear encoder.
    \textbf{(Left)} Time evolution of $N_\Phi$, $N_\Psi$, and $N_\times$.
    \textbf{(Right)} Asymmetry of the projection head $\Wbf$ (measured by the relative error of $\Wbf - \Wbf^\top$) and non-commutativity of $\Fbf$ and $\Wbf$ (measured by the relative error of the commutator $[\Fbf, \Wbf]$).
  }
  \label{figure:symmetry_comm_cifar}
\end{figure}

\begin{figure}[t]
  \centering
  \includegraphics[width=0.40\textwidth]{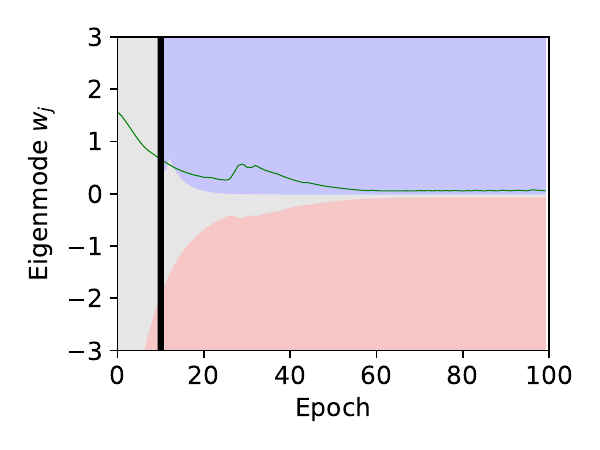}
  \includegraphics[width=0.40\textwidth]{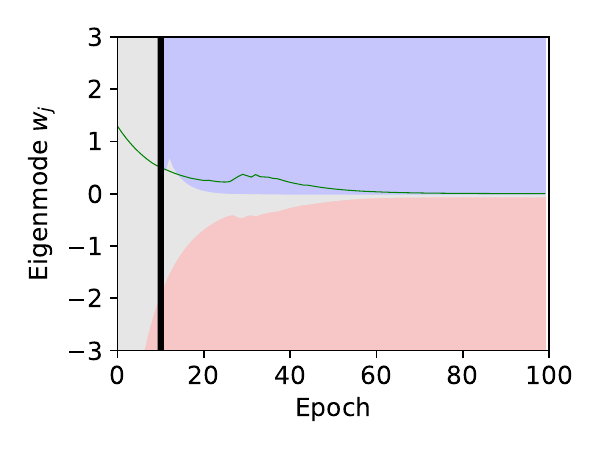}
  \includegraphics[width=0.40\textwidth]{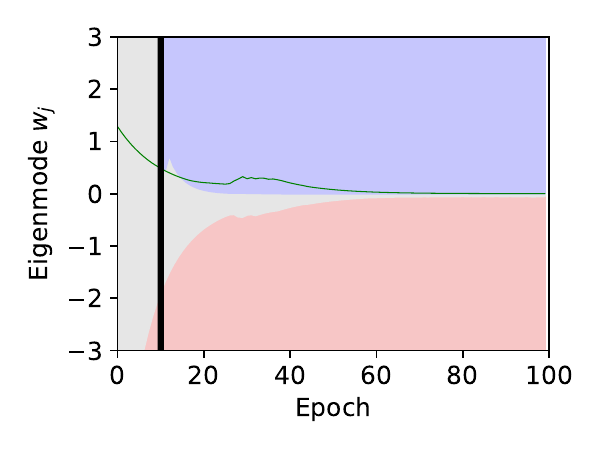}
  \includegraphics[width=0.40\textwidth]{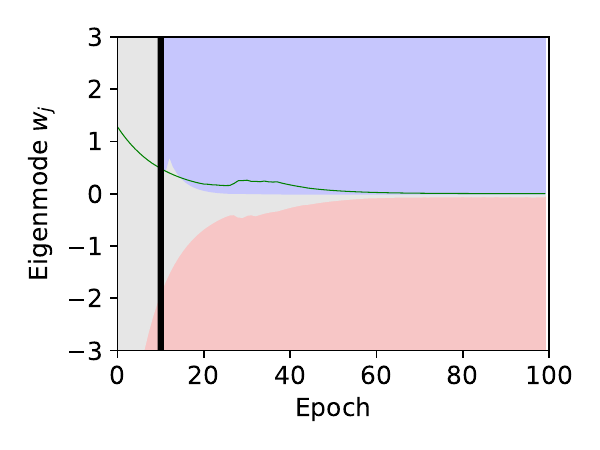}
  \caption{
    The eigenvalues of the projection head $w_j$ are plotted (trained with the nonlinear encoder), with background colors illustrating three intervals where \colorbox{red!20}{\raisebox{0pt}[0.8ex]{$w_j$ diverges}}, \colorbox{black!20}{\raisebox{0pt}[0.8ex]{$w_j$ collapses}}, and \colorbox{blue!20}{\raisebox{0pt}[0.8ex]{$w_j$ stably converges}} at each epoch. Each color corresponds to those in \cref{figure:regime}. The vertical black line indicates the bifurcates from Collapse ($\mathrm{epoch} < 10$) to Acute ($\mathrm{epoch} > 10$).
    \textbf{(Top left)} $j=1$ (the largest eigenvalue);
    \textbf{(Top center)} $j=2$ (the second largest eigenvalue);
    \textbf{(Top center)} $j=3$ (the third largest eigenvalue);
    \textbf{(Bottom left)} $j=4$ (the fourth largest eigenvalue);
  }
  \label{figure:each_eigenvalue_cifar}
\end{figure}

Here, we complement our analysis by conducting the numerical simulation of the SimSiam model using a nonlinear encoder.
As in \cref{section:experiments}, we use the official implementation of SimSiam.
The implementation differences from the official code are listed below:
\begin{itemize}
  \item Dataset: CIFAR-10
  \item The feature encoder: ResNet-18, but the last fully-connected layers being replaced with linear $\Phibf$
  \item The projection head: linear $\Wbf \in \Rbb^{2048 \times 2048}$ without bias ($h = 2048$)
  \item Parameter initialization: following \cref{assumption:initialization} and $\Wbf$ are symmetrized by $(\Wbf + \Wbf^\top)/2$
  \item Optimizer: the momentum SGD with the initial learning rate $0.005$
  \item Regularization strength: $\rho = 0.008$
  \item Epochs: $100$
\end{itemize}
We used the same data augmentation applied to the ImageNet dataset in the official implementation.
The other details remain to be the same as the official implementation.

To see how the nonlinear setup aligns with \cref{assumption:w_symmetric} (symmetry of $\Wbf$), and \cref{assumption:equilibrium_commutator} (commutativity of $\Wbf$ and $\Fbf$), we show them in \cref{figure:symmetry_comm_cifar}.
The norm parameters $N_\Phi$, $N_\Psi$, and $N_\times$ gradually shrink.
During the training epochs, $\Wbf$ becomes relatively asymmetric, but converges to a symmetric matrix.
This point needs to be carefully addressed in future work.
We can suppose that $\Wbf$ and $\Fbf$ remain to be commutative.

The time evolution of the eigenvalues of the linear projection head $\Wbf$ is shown in \cref{figure:eigenvalue_histogram_cifar}, and each eigenvalue ($j = 1, 2, 3, 4$) is shown in \cref{figure:each_eigenvalue_cifar}.
Each background color in \cref{figure:each_eigenvalue_cifar} indicates whether $w_j$ diverges (red), collapses (gray), and stably converges (blue).
The boundaries of these intervals are computed by numerical root finding of the $w_j$-dynamics \eqref{equation:dynamics_invariant_parabola}.
We observe that only a few number of eigenvalues remain to be non-zero while most of them degenerate to zero; general trend observed in the synthetic case using the linear encoder (\cref{section:regime_shift}).
Moreover, we can see that the initial Collapse regime ($\mathrm{epoch} < 10$) is lifted to the Acute regime ($\mathrm{epoch} > 10$) in \cref{figure:each_eigenvalue_cifar}.
The (non-zero) eigenvalues eventually converge to the values in the (blue) stable interval.

\end{document}